\documentclass[11pt]{article}

\usepackage[preprint]{acl}

\usepackage{times}
\usepackage{latexsym}

\usepackage[T1]{fontenc}

\usepackage[utf8]{inputenc}

\usepackage{microtype}

\usepackage{inconsolata}

\usepackage{graphicx}

\usepackage[most]{tcolorbox}

\usepackage{amsmath,amssymb,amsthm}
\usepackage{algorithm}
\usepackage{algpseudocode}
\usepackage{booktabs}
\usepackage{tabularx}

\newtheorem{proposition}{Proposition}
\newtheorem{corollary}{Corollary}
\usepackage{multirow}
\usepackage[table]{xcolor}
\usepackage{placeins}
\usepackage{CJKutf8}

\definecolor{promptbg}{RGB}{245,245,245}
\definecolor{oursbg}{RGB}{240,247,240}

\newcommand{\CCG}{\cellcolor{gray!15}}
\newcommand{\ours}{BPD}

\newtcolorbox{paperbox}{
  colback=gray!8,
  colframe=gray!60,
  boxrule=0pt,
  leftrule=1.8mm,
  arc=0pt,
  left=8pt,
  right=8pt,
  top=6pt,
  bottom=6pt,
}

\title{Balance-Aware Pivot Selection for Parallel Diffusion LLM Decoding}

\author{
 \textbf{Yejin Lee}\thanks{Equal contribution.},
 \textbf{Hyundong Jin}\footnotemark[1],
 \textbf{Yo-Sub Han}\thanks{Corresponding author.}
 \\
 Yonsei University, Seoul, Republic of Korea,
 \\
 \texttt{\{%
   \href{mailto:ssgyejin@yonsei.ac.kr}{ssgyejin},%
   \href{mailto:hsan@yonsei.ac.kr}{tuzi04},%
   \href{mailto:emmous@yonsei.ac.kr}{emmous}%
 \}@yonsei.ac.kr}
}

\begin{document}
\maketitle

\begin{abstract}
Diffusion language models enable parallel token generation,
offering a promising alternative 
to sequential autoregressive decoding.
Existing acceleration methods use predictive confidence
to enable parallel token commitment
while preserving generation quality.
However, confidence-based selection prioritizes
reliable token predictions without explicitly considering
how token positions shape the remaining decoding schedule.
Moreover, such ordering does not always
reflect inter-token dependencies,
suggesting that confidence need not be
the sole criterion for token selection.
Motivated by this limitation, we propose Balance-Aware Pivot Decoding~(\ours{}),
which favors central pivots within contiguous masked intervals
and decodes the selected positions in parallel.
Balanced pivots yield logarithmic decoding depth in the interval length.
Our practical method accounts for model uncertainty by combining
confidence-based prefilling with confidence-thresholded center selection.
We evaluate the resulting quality--efficiency tradeoff on LLaDA and Dream
across five mathematical reasoning, code generation, and instruction-following benchmarks. Our implementation is available at \url{https://anonymous.4open.science/r/DLM_random_pivot-EE62/}
\end{abstract}

\section{Introduction}
\label{bpd:main:sec:introduction}

Diffusion language models offer an alternative 
to autoregressive generation 
by iteratively recovering masked tokens 
using bidirectional context~\citep{nie2025llada,ye2025dream}.
Their ability to predict multiple positions 
in a single forward pass
can reduce the number of forward passes required for generation.
Realizing this potential requires an effective decoding strategy 
that determines how many tokens to commit
and which positions to select at each step~\citep{KimSKKC2025,ChenCL2026}.
Aggressive parallel commitment can compromise generation quality,
whereas conservative schedules leave much of 
the available parallelism unused~\citep{ZhongGZLDJSLZLZ2026}.

Recent methods address this trade-off through confidence-based selection,
caching, and iterative revision.
Fast-dLLM~\citep{WuZXLDZLHX26} combines confidence-aware parallel decoding
with KV-cache reuse, reducing both the number of decoding iterations
and the computation required per iteration.
WINO~\citep{HongYYHZZWY26} allows uncertain commitments to be revised
through a parallel drafting and verification process.
Hierarchy-dLLM~\citep{QiDZWJZ26} introduces a divide-and-conquer
structure that partitions unresolved positions into contiguous
sub-decoding areas and processes them in parallel.
These approaches demonstrate the effectiveness of
confidence-guided token selection and parallel decoding.

\begin{figure}[t]
    \centering
    \includegraphics[width=\columnwidth]{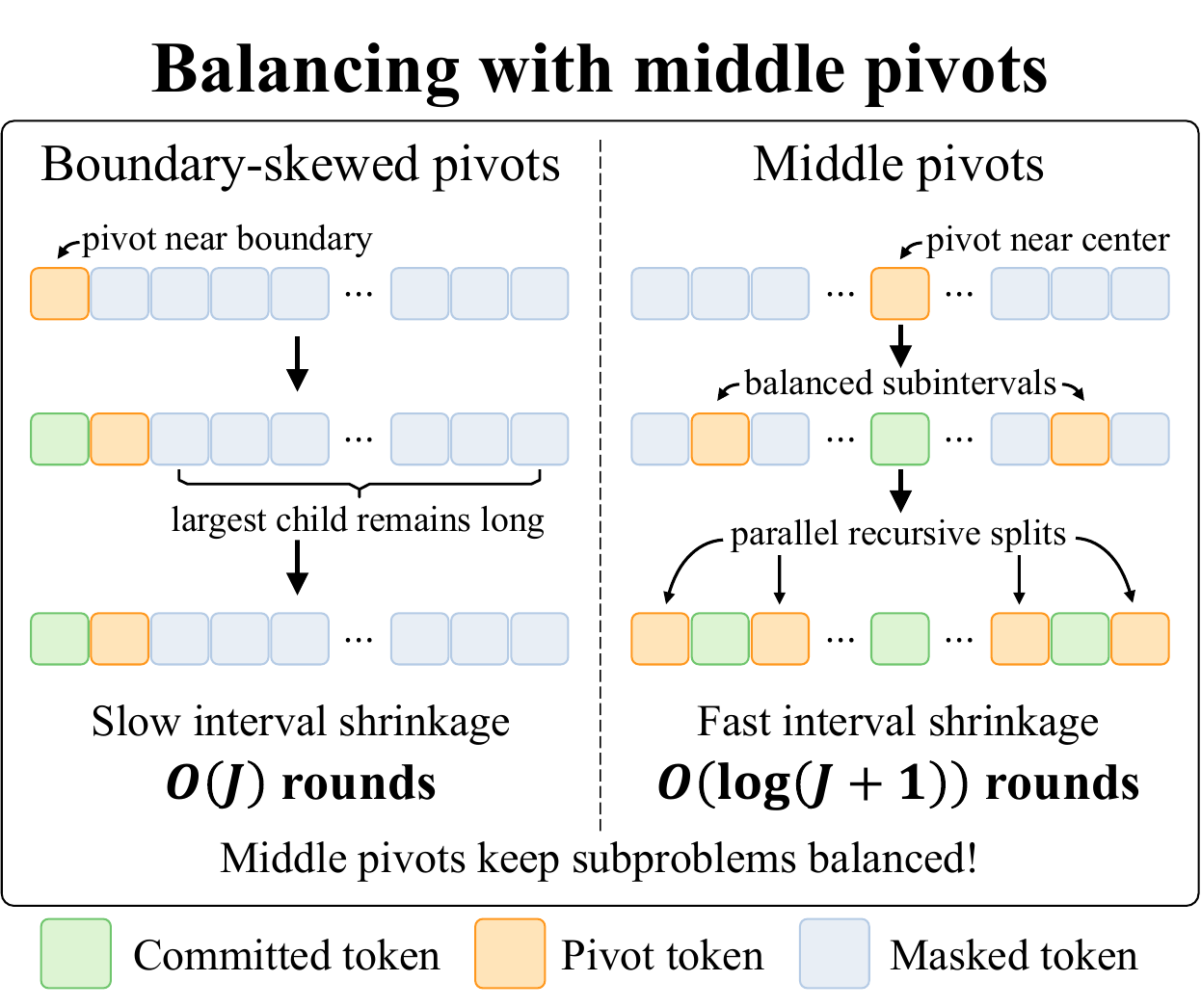}
    \caption{Motivation for position-aware unmasking strategy.
    Center-based selection promotes balanced partitions,
    unlike confidence-based selection that may leave
    long masked intervals.}
    \label{bpd:main:fig:motivation}
\end{figure}

However, confidence-based selection measures 
the reliability of individual token predictions, 
without explicitly accounting for the remaining decoding process. 
In threshold-based decoding, the number of committed tokens
depends on current predictions,
providing limited control over the remaining decoding rounds.
Moreover, recent studies suggest that confidence-based ordering
does not always align with token dependencies~\citep{KimN2026},
and that alternative unmasking orders can improve
decoding~\citep{LeeKPP2026}.
These observations motivate considering additional criteria
beyond predictive confidence when determining the decoding order. 
Rather than replacing confidence, 
we explore whether a simple positional strategy can provide 
a more structured decoding schedule 
while retaining the benefits of confidence-based token selection.

We propose \emph{Balance-Aware Pivot Decoding}~(\ours{}),
which introduces a structural preference into
confidence-aware token selection.
\ours{} favors pivots near the centers of contiguous masked intervals,
progressively dividing unresolved positions into smaller intervals
that can be processed in parallel,
as illustrated in Figure~\ref{bpd:main:fig:motivation}.
We show that balanced pivot selection
yields $O(\log(J+1))$ decoding rounds
for an interval of length $J$.
To account for prediction uncertainty,
\ours{} incorporates confidence-based prefilling
and thresholded pivot selection.
We empirically examine the resulting pivot balance
and quality--efficiency trade-off.

Our contributions are summarized as follows:

\begin{itemize}
    \item We introduce \ours{}, combining center-based pivot selection
    with confidence-aware decoding to exploit parallelism.

    \item We characterize the decoding depth of
    balanced pivot selection and analyze how
    confidence-aware pivoting affects the
    applicability of the theoretical bound.

    \item We evaluate generation quality, decoding efficiency,
    and pivot balance on LLaDA and Dream across five
    reasoning, code generation, and instruction-following benchmarks.
\end{itemize}



\section{Related Work}
\label{bpd:main:sec:related_work}
\paragraph{Diffusion Language Models}
Diffusion models, originally developed for continuous data
~\citep{HoJA20Diffusion,nichol2021improved},
have been extended to text generation through 
continuous~\citep{LiTGLH22Continu1,GongLF0K23Continu}
and discrete formulations~\citep{AustinJHTB21Discret1,HeSTWHQ23Discret}.
Masked diffusion language models~(DLMs) enable iterative
token refinement with flexible generation orders.
\citet{SahooASGMCRK24mdlm} introduced a simplified
training objective for masked diffusion,
while large-scale models such as
LLaDA~\citep{nie2025llada} and Dream~\citep{ye2025dream}
have demonstrated competitive performance
on reasoning and code generation tasks.
Further advances include preference
optimization~\citep{ZhuWNZWZLWL2026},
anchor-based conditioning~\citep{RoutCS25ADLM},
and iterative remasking~\citep{wang2025ReMDM}.

\paragraph{Parallel Decoding}
Parallel decoding exploits the ability of DLMs
to predict multiple masked tokens in a single forward pass,
reducing decoding iterations at the potential cost
of generation quality due to inter-token dependencies.
Fast-dLLM~\citep{WuZXLDZLHX26}
combines confidence-thresholded parallel decoding
with KV caching, allowing multiple high-confidence tokens
to be committed while reusing previously computed representations.
WINO~\citep{HongYYHZZWY26}
introduces a parallel drafting and verification process,
where uncertain commitments can be remasked
and refined in subsequent iterations.
Hierarchy-dLLM~\citep{QiDZWJZ26}
adopts a divide-and-conquer approach that recursively
partitions unresolved positions into sub-decoding areas,
using confidence-based selection within each area
to enable parallel progress.
It also supports remasking unreliable predictions.
These approaches improve decoding efficiency through
parallel commitment, computation reuse, or iterative revision,
while relying primarily on predictive confidence
to determine which tokens to commit.

\paragraph{Token Ordering and Structured Decoding}
Recent studies explore decoding criteria
beyond individual token confidence.
DAPD~\citep{KimJJ26DPAD}
constructs attention-based dependency graphs
to select weakly dependent tokens for parallel updates,
while DOS~\citep{ZhouHH26DOS}
uses attention-derived dependencies to guide decoding.
\citet{KimN2026} further show that confidence-guided
ordering can diverge from logical dependencies
in complex reasoning tasks.
In concurrent work, \citet{BansalCKSS2026}
investigate parallel sampling through graph random walks
and introduce recursive bisection samplers
with logarithmic parallel depth.
Their score-guided variant selects low-uncertainty pivots
within the middle half of each interval.
In contrast, \ours{} selects the most central position
among confidence-eligible candidates
and incorporates an initial confidence-based prefilling phase.
This design favors balanced splits
without imposing a fixed central-region constraint.

\section{Method}
\label{bpd:main:sec:method}

We introduce \emph{Balance-Aware Pivot Decoding}~(\ours{}),
a parallel decoding strategy that integrates positional balance with predictive confidence.
As illustrated in Figure~\ref{bpd:main:fig:overview},
\ours{} first commits high-confidence tokens in parallel
and then selects central pivots with in the remaining masked intervals to reduce decoding depth.
We first analyze the theoretical complexity of balanced pivot selection
and then describe our confidence-aware decoding procedure.

\subsection{Balance-Aware Pivot Decoding}
\label{bpd:main:ssec:bpd}

Given a prompt $u$, we initialize a generation canvas $y=\mathtt{MASK}^{G}$ 
and process blocks of size $B$ from left to right, 
assuming that $B$ divides $G$.
At iteration $t$ within a block, 
let $\mathcal M_t$ be its masked positions 
and $\mathcal I_t$ the maximal contiguous intervals in $\mathcal M_t$. 
One model forward pass produces distributions $p_i^{(t)}$ 
for the current block. 
Excluding $\mathtt{MASK}$ from the predictive vocabulary, 
we define
\begin{equation}
    \widehat y_i^{(t)}=\arg\max_v p_i^{(t)}(v),
    \quad c_i^{(t)}=p_i^{(t)}(\widehat y_i^{(t)}).
\end{equation}
The basic \ours{} rule selects
\begin{equation}
    j_t(I)=\arg\min_{i\in I}\left|i-\frac{a+b}{2}\right|,
    \quad I=[a,b]\in\mathcal I_t,
    \label{bpd:main:eq:midpoint}
\end{equation}
with positional ties resolved to the left. 
All selected predictions are committed simultaneously 
using the same forward output. 
A pivot $j$ replaces $[a,b]$ 
by its nonempty children $[a,j-1]$ and $[j+1,b]$. 
Committed tokens remain fixed.

The number of sequential rounds is determined 
by the deepest branch of this interval-splitting process. 
Since both children advance in the same subsequent rounds, 
the relevant recurrence is
\begin{equation}
    D(I)=1+\max\{D(I_{\mathrm L}),D(I_{\mathrm R})\},
    \quad D(\emptyset)=0,
    \label{bpd:main:eq:depth}
\end{equation}
for a realized splitting tree. 
Thus exact midpoint selection is not necessary.
Any pivot that contracts both children by a fixed factor 
gives logarithmic depth.

\begin{figure*}
    \centering
    \includegraphics[width=\linewidth]{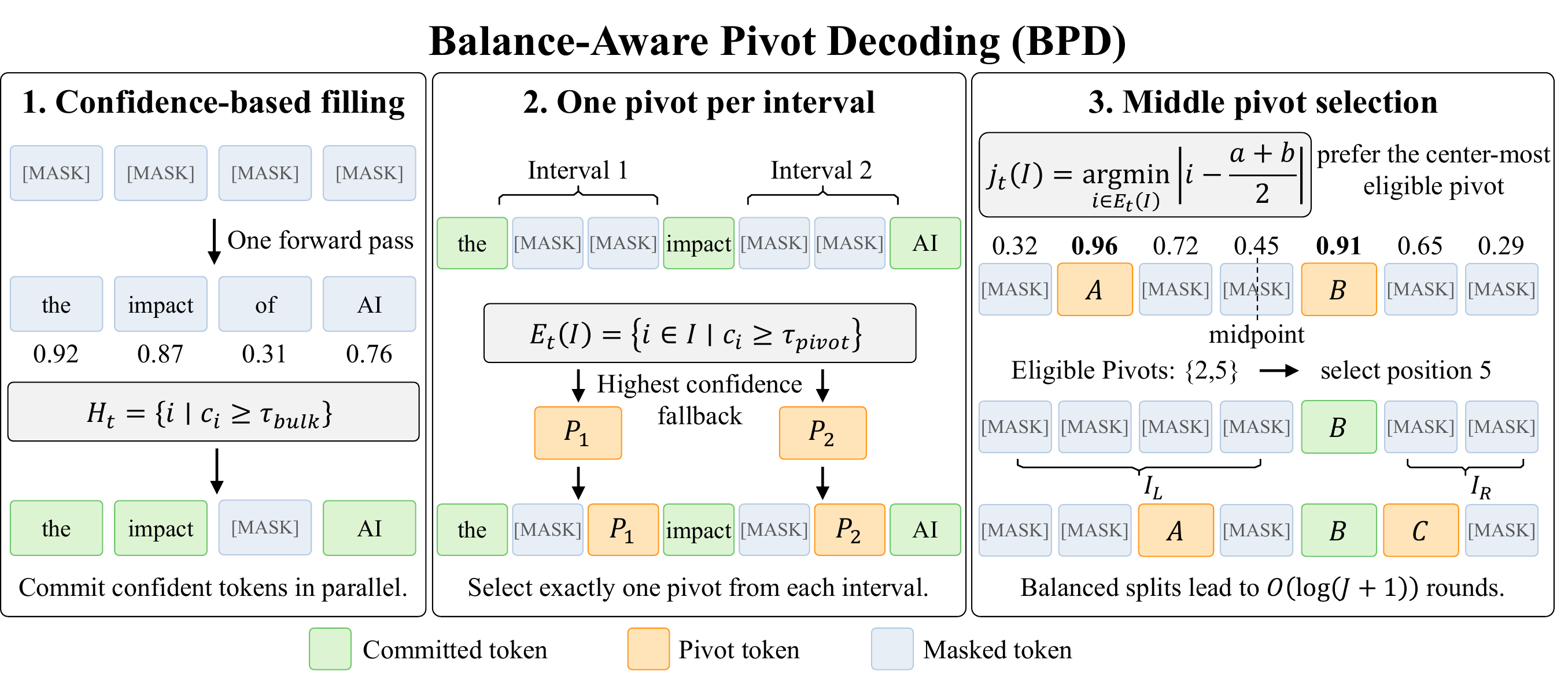}
    \caption{Overview of Balance-Aware Pivot Decoding~(\ours{}).
        \ours{} first commits high-confidence tokens in parallel,
        then selects one pivot per remaining masked interval.
        It favors the most central confidence-eligible position
        to promote balanced splits and reduce decoding depth.}
    \label{bpd:main:fig:overview}
\end{figure*}

\begin{proposition}[Balanced-pivot round complexity]
\label{bpd:main:prop:balanced}
Consider a masked interval of length $J\ge1$. Suppose every round commits one pivot in every nonempty interval and, for a fixed $\rho\in[1/2,1)$, every pivot satisfies
\begin{equation}
    \max\{|I_{\mathrm L}|,|I_{\mathrm R}|\}\leq\rho|I|.
    \label{bpd:main:eq:balance}
\end{equation}
Then decoding finishes within
\begin{equation}
    T(J)\leq\left\lfloor\frac{\log J}{\log(1/\rho)}\right\rfloor+1
    =O(\log(J+1))
    \label{bpd:main:eq:bound}
\end{equation}
rounds. In particular, pivots that leave neither child larger than $3|I|/4$ satisfy this bound with $\rho=3/4$.
\end{proposition}

Exact midpoint selection completes an interval
in $\lceil\log_2(J+1)\rceil$ rounds.
Proposition~\ref{bpd:main:prop:balanced} extends this result
to any consistently balanced pivot rule,
providing a worst-case depth bound.
An expected-depth extension for probabilistically
balanced pivots is given in Appendix~\ref{bpd:app:sec:proofs}.
These bounds concern sequential decoding rounds,
not computational cost or generation quality,
and do not assume independence between masked intervals.

\subsection{Confidence-Based Prefilling}
\label{bpd:main:ssec:prefilling}

Pure midpoint selection ignores predictive confidence
and may commit unreliable tokens during early decoding step.
To address this limitation,
we introduce confidence-based prefilling,
which commit high-confidence predictions in parallel 
before pivot selection.
During the first $K$ iterations of each block,
we select tokens according to their predictive confidence.
At iteration $t<K$, define
\begin{equation}
    H_t=\{i\in\mathcal M_t:c_i^{(t)}\geq\tau_{\mathrm{bulk}}\}.
    \label{bpd:main:eq:bulk}
\end{equation}
If $H_t\neq\emptyset$, we commit all predictions 
in $H_t$ simultaneously. 
These commitments can divide the remaining masks 
into several shorter intervals. 
If $H_t=\emptyset$, we immediately apply 
the interval-wise pivot rule below. 
The probabilities and candidate sets are recomputed 
at every iteration.
This process provides reliable contextual information
for subsequent pivot selection.

\subsection{Confidence-Threshold Pivoting}
\label{bpd:main:ssec:threshold}

A geometrically central token may have low predictive confidence. 
We therefore retain the center preference 
while restricting it to candidates above a threshold. 
For each active interval $I=[a,b]$, let
\begin{equation}
    E_t(I)=\{i\in I:c_i^{(t)}\geq\tau_{\mathrm{pivot}}\}.
\end{equation}
We select
\begin{equation}
    j_t(I)=
    \begin{cases}
        \displaystyle\arg\min_{i\in E_t(I)}\left|i-\frac{a+b}{2}\right|,
            &E_t(I)\neq\emptyset,\\[4pt]
        \displaystyle\arg\max_{i\in I}c_i^{(t)},
            &E_t(I)=\emptyset.
    \end{cases}
    \label{bpd:main:eq:practical_pivot}
\end{equation}
Ties are resolved to the left. 
This rule commits exactly one token per interval 
and always makes progress. 
Let $P_t=\{j_t(I)\mid I\in\mathcal I_t\}$ denote
the selected pivots. Combining both techniques gives
\begin{equation}
    \begin{gathered}  
        S_t=
        \begin{cases}
            H_t,&t<K\text{ and }H_t\neq\emptyset,\\
            P_t,&\text{otherwise},
        \end{cases}\\[4pt]
        y_i\gets\widehat y_i^{(t)}\quad(i\in S_t).
    \end{gathered}
    \label{bpd:main:eq:combined}
\end{equation}

The confidence threshold allows \ours{} to avoid
uncertain central predictions.
To connect this rule with Proposition~\ref{bpd:main:prop:balanced},
let $A_\rho(I)$ denote the positions in interval $I$
that satisfy the balance condition in Eq.~\eqref{bpd:main:eq:balance}.
Since these positions form a central region,
the selected pivot is guaranteed to be balanced whenever
$E_t(I)\cap A_\rho(I)\neq\emptyset$.
Otherwise, the selected pivot may lie near an endpoint,
and the logarithmic bound is not guaranteed.
Without the balance condition,
the worst-case decoding depth is $B$ rounds per block.

\paragraph{Prefix KV caching.}
We employ approximate prefix KV caching
to reduce computation across decoding iterations.
At the first inference of each block,
we obtain its predictions and the prefix KV cache
from a single full-sequence forward pass.
Subsequent iterations reuse the fixed prefix cache
while recomputing the suffix.
This reduces computation per round
without affecting the interval-depth analysis.
For a fixed $K$, if all subsequent pivots
satisfy Eq.~\eqref{bpd:main:eq:balance},
each block requires $O(K+\log(B+1))$ rounds,
yielding $O((G/B)(K+\log(B+1)))$
forward passes over the full generation.

\section{Experiment}
\label{bpd:main:sec:experiment}

\begin{table*}[t!]
\centering
\small
\setlength{\tabcolsep}{4.5pt}
\begin{tabular}{llccccccccc}
\toprule
\multirow{2}{*}[-2pt]{Dataset}
& \multirow{2}{*}[-2pt]{Method}
& \multicolumn{3}{c}{LLaDA-8B-Instruct}
& \multicolumn{3}{c}{LLaDA-1.5}
& \multicolumn{3}{c}{Dream-7B} \\
\cmidrule(lr){3-5}
\cmidrule(lr){6-8}
\cmidrule(lr){9-11}
& & Score $\uparrow$
  & TPS $\uparrow$
  & QTG $\uparrow$
  & Score $\uparrow$
  & TPS $\uparrow$
  & QTG $\uparrow$
  & Score $\uparrow$
  & TPS $\uparrow$
  & QTG $\uparrow$ \\
\midrule

\multirow{7}{*}{GSM8K}
& Vanilla
& 81.65 & 12.02 & 1.00$\times$
& 81.88 & 11.92 & 1.00$\times$
& 82.87 & 14.62 & 1.00$\times$ \\
& Hierarchy-dLLM
& 77.94 & 52.87 & 4.20$\times$
& 79.45 & 54.51 & 4.44$\times$
& 73.54 & 73.43 & 4.46$\times$ \\
& WINO
& 80.29 & 46.24 & 3.78$\times$
& 79.53 & 46.46 & 3.79$\times$
& 80.21 & 55.24 & 3.66$\times$ \\
& Fast-dLLM
& 81.50 & 52.83 & \underline{4.39}$\times$
& 81.58 & 54.96 & \underline{4.59}$\times$
& 80.52 & 74.60 & \underline{4.96}$\times$ \\
& DAPD
& 80.44 & 36.01 & 2.95$\times$
& 82.03 & 37.69 & 3.17$\times$
& 82.56 & 41.84 & 2.85$\times$ \\
& DOS
& 79.98 & 22.32 & 1.82$\times$
& 82.18 & 23.39 & 1.97$\times$
& 82.49 & 32.75 & 2.23$\times$ \\
& \CCG{}Ours
& \CCG{}79.98 & \CCG{}54.50
& \CCG{} \textbf{4.44}$\times$
& \CCG{}80.59 & \CCG{}56.10
& \CCG{} \textbf{4.63}$\times$
& \CCG{}79.91 & \CCG{}82.87
& \CCG{} \textbf{5.47}$\times$ \\
\midrule

\multirow{7}{*}{MATH-500}
& Vanilla
& 38.60 & 12.13 & 1.00$\times$
& 37.40 & 12.07 & 1.00$\times$
& 42.00 & 14.70 & 1.00$\times$ \\
& Hierarchy-dLLM
& 36.20 & 40.94 & 3.17$\times$
& 36.80 & 41.24 & 3.36$\times$
& 26.40 & 48.22 & 2.06$\times$ \\
& WINO
& 37.20 & 28.08 & 2.23$\times$
& 38.20 & 27.97 & 2.37$\times$
& 40.20 & 34.97 & 2.28$\times$ \\
& Fast-dLLM
& 37.00 & 41.06 & \underline{3.24}$\times$
& 36.20 & 42.02 & \underline{3.37}$\times$
& 42.60 & 54.40 & \underline{3.75}$\times$ \\
& DAPD
& 37.80 & 26.63 & 2.15$\times$
& 36.60 & 22.78 & 1.85$\times$
& 39.40 & 28.93 & 1.85$\times$ \\
& DOS
& 34.60 & 17.27 & 1.28$\times$
& 35.80 & 18.05 & 1.43$\times$
& 39.60 & 21.19 & 1.36$\times$ \\
& \CCG{}Ours
& \CCG{}37.00 & \CCG{}44.20
& \CCG{}\textbf{3.49}$\times$
& \CCG{}37.00 & \CCG{}44.95
& \CCG{}\textbf{3.68}$\times$
& \CCG{}40.40 & \CCG{}62.52
& \CCG{}\textbf{4.09}$\times$ \\
\midrule

\multirow{7}{*}{HumanEval}
& Vanilla
& 35.98 & 10.44 & 1.00$\times$
& 39.63 & 10.34 & 1.00$\times$
& 59.76 & 11.79 & 1.00$\times$ \\
& Hierarchy-dLLM
& 32.93 & 72.53 & 6.36$\times$
& 37.20 & 67.81 & 6.16$\times$
& 32.32 & 83.43 & 3.83$\times$ \\
& WINO
& 39.63 & 43.49 & 4.59$\times$
& 38.41 & 39.75 & 3.73$\times$
& 57.32 & 51.30 & 4.17$\times$ \\
& Fast-dLLM
& 39.02 & 79.06 & \textbf{8.21$\times$}
& 38.41 & 78.54 & \textbf{7.36$\times$}
& 60.37 & 91.39 & \underline{7.83}$\times$ \\
& DAPD
& 37.80 & 47.33 & 4.76$\times$
& 37.80 & 46.23 & 4.26$\times$
& 57.32 & 39.38 & 3.20$\times$ \\
& DOS
& 34.76 & 26.93 & 2.49$\times$
& 37.80 & 27.76 & 2.56$\times$
& 57.32 & 34.08 & 2.77$\times$ \\
& \CCG{}Ours
& \CCG{}35.98 & \CCG{}78.29
& \CCG{}\underline{7.50}$\times$
& \CCG{}37.20 & \CCG{}77.83
& \CCG{}\underline{7.07}$\times$
& \CCG{}54.88 & \CCG{}107.61
& \CCG{}\textbf{8.38$\times$} \\
\midrule

\multirow{7}{*}{MBPP}
& Vanilla
& 39.60 & 10.78 & 1.00$\times$
& 40.00 & 10.74 & 1.00$\times$
& 58.00 & 12.69 & 1.00$\times$ \\
& Hierarchy-dLLM
& 36.80 & 49.38 & 4.26$\times$
& 36.20 & 49.30 & 4.15$\times$
& 49.20 & 124.89 & 8.35$\times$ \\
& WINO
& 38.80 & 26.50 & 2.41$\times$
& 38.60 & 26.71 & 2.40$\times$
& 56.40 & 56.34 & 4.32$\times$ \\
& Fast-dLLM
& 38.00 & 52.01 & \underline{4.63}$\times$
& 37.80 & 52.89 & \underline{4.65}$\times$
& 57.60 & 112.33 & \underline{8.79}$\times$ \\
& DAPD
& 38.80 & 30.74 & 2.79$\times$
& 38.00 & 32.03 & 2.83$\times$
& 56.20 & 69.43 & 5.30$\times$ \\
& DOS
& 40.20 & 20.58 & 1.94$\times$
& 39.80 & 21.28 & 1.97$\times$
& 57.80 & 54.23 & 4.26$\times$ \\
& \CCG{}Ours
& \CCG{}38.20 & \CCG{}55.90
& \CCG{}\textbf{5.00}$\times$
& \CCG{}38.60 & \CCG{}56.69
& \CCG{}\textbf{5.09}$\times$
& \CCG{}53.00 & \CCG{}131.36
& \CCG{}\textbf{9.46$\times$} \\
\midrule

\multirow{7}{*}{IFEval}
& Vanilla
& 56.01 & 12.26 & 1.00$\times$
& 58.60 & 12.98 & 1.00$\times$
& 51.39 & 16.24 & 1.00$\times$ \\
& Hierarchy-dLLM
& 46.40 & 35.88 & \underline{2.42}$\times$
& 50.28 & 35.97 & \underline{2.38}$\times$
& 31.61 & 58.51 & 2.22$\times$ \\
& WINO
& 55.82 & 21.23 & 1.73$\times$
& 57.12 & 21.80 & 1.64$\times$
& 50.83 & 31.70 & 1.93$\times$ \\
& Fast-dLLM
& 53.79 & 30.09 & 2.36$\times$
& 57.49 & 31.41 & 2.37$\times$
& 52.31 & 47.52 & 2.98$\times$ \\
& DAPD
& 58.23 & 25.91 & 2.20$\times$
& 58.41 & 25.59 & 1.97$\times$
& 40.67 & 67.35 & \underline{3.28}$\times$ \\
& DOS
& 59.89 & 19.77 & 1.72$\times$
& 60.44 & 19.52 & 1.55$\times$
& 56.93 & 25.17 & 1.72$\times$ \\
& \CCG{}Ours
& \CCG{}51.20 & \CCG{}43.66
& \CCG{}\textbf{3.26}$\times$
& \CCG{}54.71 & \CCG{}43.48
& \CCG{}\textbf{3.13}$\times$
& \CCG{}51.57 & \CCG{}63.94
& \CCG{}\textbf{3.95$\times$} \\
\bottomrule
\end{tabular}
\caption{
Main results on five benchmarks using LLaDA-8B-Instruct,
LLaDA-1.5, and Dream-7B.
We report generation quality (Score), throughput measured
in tokens per second (TPS), and Quality-Throughput Gain (QTG).
QTG is the product of relative throughput and relative
task performance with respect to vanilla.
}
\label{bpd:main:tab:main_results}
\end{table*}

\subsection{Experimental Setting}
\label{bpd:main:ssec:experimental_setting}
We evaluate our method on three diffusion language models: LLaDA-8B-Instruct, LLaDA-1.5, and Dream-7B.
We conduct all experiments on a single NVIDIA RTX A6000 GPU with 48 GB of memory.
We use semi-autoregressive decoding with a block size of 32 and a maximum generation length of 256 tokens.
We set both confidence threshold $\tau_{\mathrm{bulk}}$ and  $\tau_{\mathrm{pivot}}$ to 0.9, and the number of prefilling steps $K$ to 3.
We use the same decoding configuration for all benchmarks unless otherwise specified.
Further implementation details are provided in
Appendix~\ref{bpd:app:ssec:datasets_metrics_prompts}.

\paragraph{Benchmarks and Baselines.}
We evaluate our method on five benchmarks: GSM8K~\citep{cobbe2021training} and MATH-500~\citep{LightmanKBEBLLS24Math500} for mathematical reasoning, 
HumanEval~\citep{chen2021humaneval} and MBPP~\citep{austin2021mbpp} for code generation, 
and IFEval~\citep{zhou2023ifeval} for instruction following. 
We compare our method against five parallel decoding baselines: Hierarchy-dLLM~\citep{QiDZWJZ26}, WINO~\citep{HongYYHZZWY26}, Fast-dLLM~\citep{WuZXLDZLHX26}, DAPD~\citep{KimJJ26DPAD}, and DOS~\citep{ZhouHH26DOS}.

\paragraph{Evaluation Metrics.}
We evaluate task performance using Pass@1 for HumanEval and MBPP and accuracy for the remaining benchmarks. 
We measure inference efficiency using tokens per second~(TPS), 
which measures the number of generated tokens per second.
Since decoding acceleration can trade generation quality
for throughput, we use Quality-Throughput Gain~(QTG)
as our primary aggregate metric to evaluate both aspects jointly.
QTG is defined as the relative throughput weighted by relative
task performance with respect to vanilla decoding:

\begin{equation}
\mathrm{QTG}
=\frac{\mathrm{TPS}_{m}}{\mathrm{TPS}_{v}}
\cdot
\frac{\mathrm{Score}_{m}}{\mathrm{Score}_{v}},
\end{equation}

where $m$ denotes the evaluated decoding method and $v$ denotes vanilla decoding, and $\mathrm{Score}$ represents Pass@1 or accuracy.

\subsection{Main Results}
\label{bpd:main:ssec:main_results}

Table~\ref{bpd:main:tab:main_results} compares
\ours{} with five accelerated decoding baselines
across three models and five benchmarks.
\ours{} achieves the highest QTG in 13 of the 15 settings 
and the highest throughput in 12 settings. 
We examine these results by decoding
method, task, and model to distinguish efficiency gains
from changes in generation quality.

\paragraph{Comparison across Decoding Methods.}
Hierarchy-dLLM provides a particularly relevant
comparison because it also partitions unresolved
positions into smaller regions, while relying
primarily on confidence for token selection.
Compared with Hierarchy-dLLM, \ours{} achieves
higher throughput in all 15 settings, with
equal or higher task scores throughout.
Fast-dLLM is a stronger competitor, particularly
on LLaDA code generation.
On HumanEval, Fast-dLLM achieves QTG values of
8.21$\times$ and 7.36$\times$ on the two LLaDA
models, compared with 7.50$\times$ and
7.07$\times$ for \ours{}.
Dependency-aware methods, including DAPD and DOS,
retain higher task scores in some settings but
generally operate at lower throughput.

\paragraph{Comparison across Datasets.}
On GSM8K and MATH-500, \ours{} remains within
three percentage points of vanilla accuracy
across all six model-benchmark settings,
while achieving approximately 3.6--5.7$\times$
higher throughput.
Code generation exhibits a more model-dependent
quality-efficiency trade-off.
On the two LLaDA models, HumanEval Pass@1
is preserved or decreases by at most
2.43 percentage points relative to vanilla.
On Dream, however, HumanEval Pass@1 decreases
from 59.76 to 54.88, and MBPP Pass@1 decreases
from 58.00 to 53.00, while throughput reaches
107.61 and 131.36 TPS, respectively.
IFEval presents a different challenge.
Although \ours{} achieves the highest QTG
for all three models, its instruction-following
scores decrease by 4.81 and 3.89 percentage
points on the two LLaDA models.
These results highlight that decoding acceleration
does not uniformly preserve generation quality,
and that throughput gains must be interpreted
alongside task-specific performance.

\paragraph{Comparison across Models.}
We use identical block sizes, prefilling schedules,
and confidence thresholds across all three models,
without model-specific tuning.
Nevertheless, the relative advantages of
different decoding strategies vary across
model families.
For instance, on Dream, \ours{} achieves
17.7\% and 16.9\% higher throughput than
Fast-dLLM on HumanEval and MBPP, respectively.
These improvements are accompanied by
Pass@1 reductions of 5.49 and 4.60 percentage
points relative to Fast-dLLM.
These results indicate that the proposed
decoding strategy transfers across different
model families, although its relative
quality-throughput advantage depends on
the model and task.

\begin{table}[t]
\centering
\small
\setlength{\tabcolsep}{3.0pt}
\begin{tabular}{lccccc}
\toprule
Method & GSM8K & MATH & H.Eval & MBPP & IFEval \\
\midrule

\multirow{2}{*}{Vanilla}
& 11.41 & 12.69 & 10.99 & 11.11 & 12.28 \\
& (81.65) & (38.60) & (35.98) & (39.60) & (56.01) \\
\midrule

\multirow{2}{*}{Hierarchy}
& 51.36 & 48.74 & 85.36 & 53.96 & 37.69 \\
& (77.94) & (36.20) & (32.93) & (36.80) & (46.40) \\
\midrule

\multirow{2}{*}{WINO}
& 48.67 & 33.84 & 50.23 & 29.55 & 21.49 \\
& (80.29) & (37.20) & (39.63) & (38.80) & (55.82) \\
\midrule

\multirow{2}{*}{Fast-dLLM}
& 54.90 & 47.69 & 86.40 & 56.50 & 31.94 \\
& (81.50) & (37.00) & (39.02) & (38.00) & (53.79) \\
\midrule

\multirow{2}{*}{DAPD}
& 37.77 & 33.50 & 56.87 & 35.20 & 28.17 \\
& (80.44) & (37.80) & (37.80) & (38.80) & (58.23) \\
\midrule

\multirow{2}{*}{DOS}
& 23.31 & 20.24 & 33.99 & 23.68 & 21.17 \\
& (79.98) & (34.60) & (34.76) & (40.20) & (59.89) \\
\midrule

\rowcolor{gray!15}
& \textbf{56.12} & \textbf{50.17} & \textbf{86.97}
& \textbf{61.32} & \textbf{46.89} \\
\rowcolor{gray!15}
\multirow{-2}{*}{\textbf{Ours}}
& (79.98) & (37.00) & (35.98) & (38.20) & (51.20) \\

\bottomrule
\end{tabular}
\caption{
Generation quality and throughput on five benchmarks
using LLaDA-8B-Instruct.
Each method reports TPS in the first row and
task performance (Score) in the second row.
}
\label{bpd:main:tab:llada_correct_throughput}
\end{table}

\subsection{Throughput on Correct Generations}
\label{bpd:main:ssec:correct_throughput}

Overall throughput may be influenced by failed generations 
across decoding methods.
We therefore additionally measure TPS on correctly
solved instances, while reporting task scores over
the complete evaluation set.
Table~\ref{bpd:main:tab:llada_correct_throughput}
shows that \ours{} achieves the highest correct-generation
throughput on all five benchmarks with LLaDA-8B-Instruct.
Compared with the next-fastest method,
\ours{} improves throughput by 8.5\% on MBPP
and 24.4\% on IFEval.
The gains are more modest on GSM8K, MATH-500,
and HumanEval, ranging from 0.7\% to 2.9\%.
These findings show that the efficiency advantage
persists on successful generations, rather than
appearing only in the aggregate results.
Because the correct-generation subsets can differ
across methods, this analysis complements the
full-set quality and throughput comparisons
rather than replacing them.

\section{Analysis}
\label{bpd:main:sec:analysis}

\subsection{Ablation Study}
\label{bpd:main:ssec:ablation_study}

We evaluate the contributions of positional pivot selection,
confidence-based prefilling, and KV caching
on HumanEval using LLaDA-8B-Instruct,
as reported in Table~\ref{bpd:main:tab:ablation_study}.
To isolate the effect of positional preference,
we replace our center-preferred pivot rule with
random and confidence-based selection variants,
while retaining the other components.
These variants increase forward steps from
55.73 to 64.20 and 71.79, respectively,
while reducing throughput from 78.29 to
68.72 and 62.34 TPS.
Both variants also achieve lower Pass@1,
supporting the benefit of positional preference.
Unlike the bisection strategies in
Section~\ref{bpd:main:ssec:pivot_balance},
these variants do not enforce a fixed
balance constraint on pivot positions.

Confidence-based prefilling contributes to both
generation quality and decoding efficiency.
Removing prefilling reduces Pass@1 from 35.98 to 33.54
and increases forward steps from 55.73 to 73.77.
Early high-confidence commitments reduce the number
of unresolved tokens and can partition long masked
intervals before subsequent pivot selection.
Prefix KV caching further improves inference efficiency.
Without caching, throughput decreases from 78.29
to 54.01 TPS, despite requiring fewer forward steps.
The noncached variant also achieves a slightly
higher Pass@1 of 37.20, suggesting that
approximate cache reuse may affect predictions.
Overall, the complete method achieves the highest
QTG of 7.50$\times$ among the evaluated variants.

\begin{table}[t]
    \centering
    \small
    \setlength{\tabcolsep}{5pt}
    \begin{tabular}{lrrrr}
        \toprule
        Variant & Pass@1 & TPS & \# Steps & QTG \\
        \midrule
        Vanilla & 35.98 & 10.44 & 256.00 & 1.00$\times$ \\
        w/o KV caching & 37.20 & 54.01 & 49.21 & 5.35$\times$\\
        w/o Prefilling & 33.54 & 59.15 & 73.77 & 5.28$\times$\\
        w/ Random Pivot & 34.15 & 68.72 & 64.20 & 6.25$\times$\\
        w/ Confident Pivot & 35.37 & 62.34 & 71.79 & 5.87$\times$\\
        \midrule
        Ours & 35.98 & 78.29 & 55.73 & 7.50$\times$\\
        \bottomrule
    \end{tabular}
    \caption{
        Ablation study on HumanEval using LLaDA-8B-Instruct to examine the contribution of each component. We report generation quality, throughput, and quality-throughput gain~(QTG).
    }
    \label{bpd:main:tab:ablation_study}
\end{table}

\subsection{Pivot Selection, Prefilling, and Balance}
\label{bpd:main:ssec:pivot_balance}

We further examine how pivot selection interacts
with confidence-based prefilling and the balance
condition underlying
Proposition~\ref{bpd:main:prop:balanced}.
Table~\ref{bpd:main:tab:pivot_balance_analysis}
compares \ours{} with three bisection strategies
from \citet{BansalCKSS2026}: Exact, Margin, and Entropy.
We additionally include a random bisection variant
that samples pivots from the balanced region.
Each strategy is evaluated with and without
confidence-based prefilling, using identical
generation settings and prefix KV caching.
Implementation details are provided in
Appendix~\ref{bpd:app:ssec:baselines}.

\paragraph{Pivot Selection and Prefilling.}
Without prefilling, \ours{} achieves the highest
Pass@1 of 33.54\% and QTG of 5.28$\times$
among the five strategies, despite having the
lowest throughput.
Exact bisection achieves higher throughput, but its Pass@1
drops to 23.17\%.
This contrast illustrates that enforcing balanced
partitions can accelerate decoding at the cost
of generation quality.
Adding prefilling improves throughput across
all five strategies, by 41.6--50.0\% for the
bisection variants and 32.4\% for \ours{}.
However, its effect on Pass@1 depends on the
pivot selection strategy.
Prefilling improves Pass@1 for Margin, Random,
and \ours{}, but reduces it for Exact and Entropy.
For example, Margin improves from 25.00\% to
29.27\%, whereas Entropy decreases from 29.88\%
to 27.44\%.
This suggests that early high-confidence commitments
interact with subsequent pivot selection rather than
uniformly improving generation quality.
With prefilling, Margin achieves the highest
QTG of 7.59$\times$, slightly exceeding
the 7.50$\times$ of \ours{}.
However, \ours{} retains the highest Pass@1
of 35.98\%, exceeding Margin by 6.71 percentage
points, while operating at lower throughput.
These results highlight the trade-off between
enforcing balanced partitions and incorporating
prediction confidence into pivot selection.

\begin{table}[t]
    \centering
    \small
    \setlength{\tabcolsep}{3pt}
    \begin{tabular}{lcccccc}
        \toprule
        \multirow{2}{*}[-2pt]{Method} & \multicolumn{3}{c}{Pivoting} & \multicolumn{3}{c}{Prefilling} \\
        \cmidrule(lr){2-4}
        \cmidrule(lr){5-7}
         & Pass@1 & TPS & QTG & Pass@1 & TPS & QTG \\
        \midrule
        Exact & 23.17 & 76.65 & 4.73 & 22.56 & 108.57 & 6.52 \\
        Margin & 25.00 & 64.96 & 4.32 & 29.27 & 97.41 & \textbf{7.59} \\
        Entropy & 29.88 & 63.53 & \underline{5.05} & 27.44 & 94.54 & 6.91 \\
        Random & 20.12 & 69.25 & 3.71 & 21.34 & 102.17 & 5.80 \\
        \midrule
        Ours & 33.54 & 59.15 & \textbf{5.28} & 35.98 & 78.29 & \underline{7.50} \\
        \bottomrule
    \end{tabular}
    \caption{
    Comparison of pivot selection strategies
    with and without confidence-based prefilling
    on LLaDA-8B-Instruct and HumanEval.
    }
    \label{bpd:main:tab:pivot_balance_analysis}
\end{table}

\begin{table}[t]
    \centering
    \small
    \setlength{\tabcolsep}{9pt}
    \begin{tabular}{lrrr}
        \toprule
        Pivot Type & \# Events & \# Balanced & BPR (\%) \\
        \midrule
        Eligible & 2,570 & 1,451 & 56.46 \\
        Fallback & 2,113 & 423 & 20.02 \\
        \midrule
        Total & 4,683 & 1,874 & 40.02 \\
        \bottomrule
    \end{tabular}
    \caption{
        Balance statistics of \ours{} pivot selections.
        The first three iterations of each block
        and intervals of length at most four are excluded.
    }
    \label{bpd:main:tab:pivot_type_breakdown}
\end{table}

\begin{figure*}[!t]
    \centering
    \includegraphics[width=0.9\linewidth]{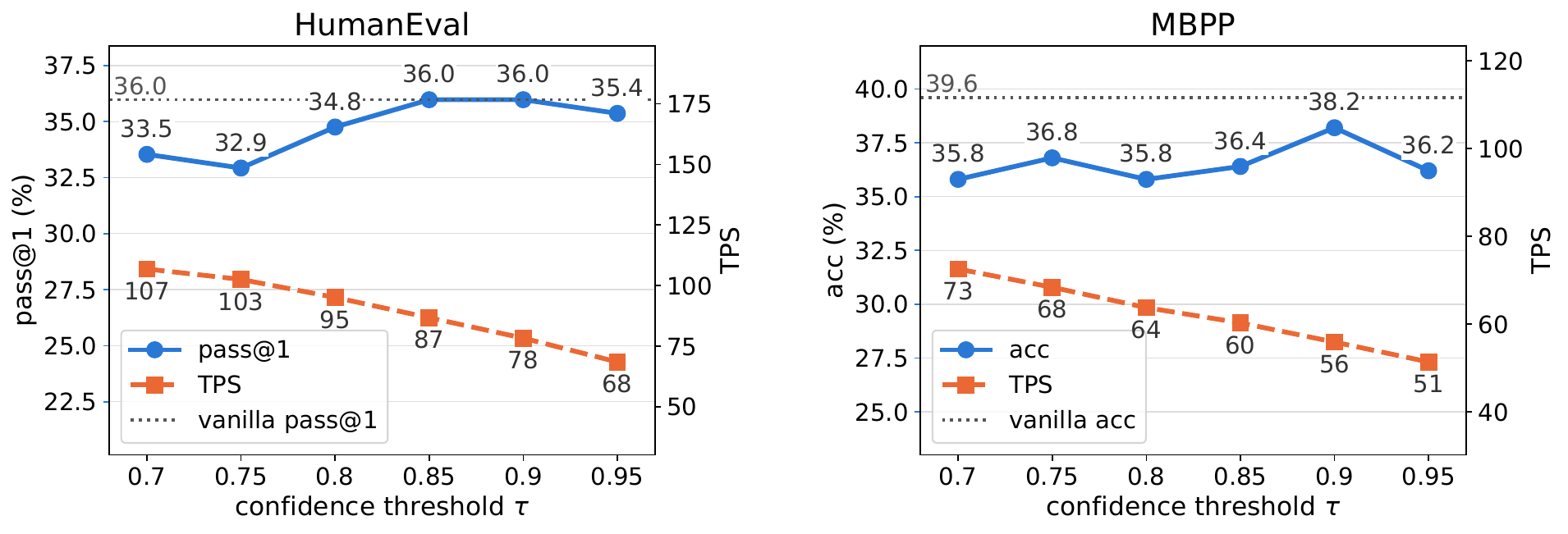}
    \caption{Effect of the confidence threshold on generation quality and throughput using LLaDA-8B-Instruct. We report Pass@1 on HumanEval and Accuracy on MBPP, along with tokens per second~(TPS). Dotted lines indicate vanilla decoding performance.}
    \label{bpd:main:fig:threshold}
\end{figure*}

\paragraph{Theoretical and Empirical Balance.}
The bisection variants enforce the $3/4$-balance
condition at every pivot selection.
Consequently, their pivoting rounds satisfy
the logarithmic depth bound in
Proposition~\ref{bpd:main:prop:balanced}.
With prefilling, the same bound applies
to the remaining intervals after the initial
$K$ iterations, yielding
$O(K+\log(R+1))$ total rounds,
where $R$ is the largest remaining interval.
In contrast, \ours{} prioritizes central positions
only among confidence-eligible candidates
and may select unbalanced fallback pivots.
Its logarithmic depth is therefore not guaranteed.

To characterize this distinction,
we measure the balanced pivot rate (BPR)
and critical-path BPR (CP-BPR), as defined in
Appendix~\ref{bpd:app:ssec:datasets_metrics_prompts}.
As shown in Table~\ref{bpd:main:tab:pivot_type_breakdown},
\ours{} achieves an overall BPR of 40.02\%.
Confidence-eligible pivots satisfy the balance
condition more frequently than fallback pivots.
The CP-BPR is 31.90\%, 8.12 percentage points
below the overall BPR.
This indicates that unbalanced selections are
more frequent along the critical paths,
which determine the number of decoding rounds
rather than the average balance across all pivots.

The probabilistic bound in
Appendix~\ref{bpd:app:sec:proofs}
requires a lower bound on the probability
of balanced selection conditional on
the decoding history, which is not implied
by the observed aggregate BPR.
Instead, the theoretical analysis establishes
the structural advantage of balanced partitioning,
while our experiments characterize the trade-off
introduced by relaxing this constraint.
Strict bisection guarantees interval contraction
but may commit uncertain predictions,
whereas \ours{} permits unbalanced splits
when confidence-eligible central pivots
are unavailable.
This flexibility allows \ours{} to achieve
higher generation quality than strictly balanced
strategies, while retaining a substantial
throughput improvement over vanilla decoding.
These results support positional balance
as a useful decoding preference rather than
an unconditional constraint.

\subsection{Effect of the Confidence Threshold}
\label{bpd:main:ssec:hyperparameter}
We investigate how the confidence threshold affects generation quality and decoding efficiency.
Figure~\ref{bpd:main:fig:threshold} shows that increasing the threshold consistently reduces throughput, while improvements in task performance remain limited.
HumanEval achieves its highest Pass@1 at thresholds of 0.85 and 0.90, whereas MBPP reaches its highest accuracy at 0.90.
Further increasing the threshold reduces performance on both benchmarks despite confidence filtering.

This behavior indicates the interaction between confidence filtering and pivot selection.
\ours{} selects the position closest to the interval center among candidates that satisfy the confidence threshold.
As the threshold increases, fewer positions remain eligible,
reducing the chance of finding a reliable pivot near the center.
This restriction can shift pivot selection away from interval centers and produce uneven partitions that require additional decoding steps.
The results suggest that excessively strict confidence filtering restricts the structural advantage of \ours{} without providing further gains in generation quality.
These findings demonstrate that the confidence threshold determines how effectively \ours{} can exploit balanced interval partitioning while maintaining prediction reliability.

\section{Conclusion}
\label{bpd:main:sec:conclusion}

We presented Balance-Aware Pivot Decoding~(\ours{}),
a training-free parallel decoding strategy for diffusion
language models that incorporates positional balance
into confidence-aware token selection.
\ours{} combines confidence-based prefilling with
center-preferred pivot selection to promote balanced
interval decomposition without forcing uncertain commitments.
We established logarithmic decoding-depth bounds
under balanced pivot selection and examined their
applicability to practical decoding.
Experiments across three diffusion language models
and five benchmarks demonstrate favorable
quality-throughput trade-offs, achieving up to
10.35$\times$ higher throughput than vanilla decoding
with a QTG of 9.46$\times$.
\ours{} achieves the highest QTG in 13 of 15 settings
and the highest throughput in 12 settings.
These findings suggest that positional balance
can serve as a useful structural preference
rather than a strict constraint for accelerating
parallel diffusion decoding.

\section*{Limitations}

Our theoretical decoding-depth bound assumes that
every selected pivot satisfies a fixed balance condition.
In practice, \ours{} incorporates confidence-based selection
and fallback mechanisms that do not guarantee this condition.
Thus, the logarithmic bound does not unconditionally
apply to the practical algorithm.
Moreover, the theoretical analysis concerns sequential
decoding rounds rather than wall-clock latency,
which also depends on computational and caching overhead.
Additionally, \ours{} does not revise committed tokens,
which may limit its robustness when uncertain
predictions are committed prematurely.



\bibliography{custom}

@inproceedings{SahooASGMCRK24mdlm,
 title = {Simple and Effective Masked Diffusion Language Models},
 author = {Sahoo, Subham Sekhar and Arriola, Marianne and Schiff, Yair and Gokaslan, Aaron and Marroquin, Edgar and Chiu, Justin T and Rush, Alexander and Kuleshov, Volodymyr},
 booktitle = {Advances in Neural Information Processing Systems, {NeurIPS}},
 pages = {130136--130184},
 year = {2024}
}

@inproceedings{nie2025llada,
  title        = {Large Language Diffusion Models},
  author       = {Shen Nie and
                  Fengqi Zhu and
                  Zebin You and
                  Xiaolu Zhang and
                  Jingyang Ou and
                  Jun Hu and
                  Jun Zhou and
                  Yankai Lin and
                  Ji{-}Rong Wen and
                  Chongxuan Li},
  booktitle    = {Advances in Neural Information Processing Systems 38: Annual Conference on Neural Information Processing Systems, {NeurIPS}},
  year         = {2025},
}

@inproceedings{ZhuWNZWZLWL2026,
  author       = {Fengqi Zhu and
                  Rongzhen Wang and
                  Shen Nie and
                  Xiaolu Zhang and
                  Chunwei Wu and
                  Jun Zhou and
                  Yankai Lin and
                  Ji{-}Rong Wen and
                  Chongxuan Li},
  title        = {LLaDA 1.5: Variance-Reduced Preference Optimization for Large Language Diffusion Models},
  booktitle    = {Proceedings of the 64th Annual Meeting of the Association for Computational Linguistics (Volume 1: Long Papers), {ACL}},
  pages        = {11425--11460},
  year         = {2026},
}

@misc{ye2025dream,
  title={Dream 7b: Diffusion Large Language Models},
  author={Ye, Jiacheng and Xie, Zhihui and Zheng, Lin and Gao, Jiahui and Wu, Zirui and Jiang, Xin and Li, Zhenguo and Kong, Lingpeng},
  howpublished={arXiv preprint arXiv:2508.15487},
  year={2025}
}

@inproceedings{wang2025ReMDM,
  title={Remasking discrete diffusion models with inference-time scaling},
  author={Wang, Guanghan and Schiff, Yair and Sahoo, Subham Sekhar and Kuleshov, Volodymyr},
  booktitle = {Advances in Neural Information Processing Systems, {NeurIPS}},
  year={2025}
}

@inproceedings{RoutCS25ADLM,
  title        = {Anchored Diffusion Language Model},
  author       = {Litu Rout and
                  Constantine Caramanis and
                  Sanjay Shakkottai},
  booktitle    = {Advances in Neural Information Processing Systems 38: Annual Conference on Neural Information Processing Systems, {NeurIPS}},
  year         =  {2025}
}

@article{KimJJ26DPAD,
  author       = {Bumjun Kim and
                  Dongjae Jeon and
                  Moongyu Jeon and
                  Albert No},
  title        = {Dependency-Aware Parallel Decoding via Attention for Diffusion LLMs},
  journal      = {ICML},
  year         = {2026},
}

@inproceedings{ZhouHH26DOS,
  author       = {Xueyu Zhou and
                  Yangrong Hu and
                  Jian Huang},
  title        = {{DOS:} Dependency-Oriented Sampler for Masked Diffusion Language Models},
  booktitle    = {Findings of the Association for Computational Linguistics, {ACL}},
  year = {2026}
}

@inproceedings{HoJA20Diffusion,
  author       = {Jonathan Ho and
                  Ajay Jain and
                  Pieter Abbeel},
  title        = {Denoising Diffusion Probabilistic Models},
  booktitle    = {Advances in Neural Information Processing Systems 33: Annual Conference on Neural Information Processing Systems 2020, {NeurIPS}},
  year = {2020}
}

@inproceedings{nichol2021improved,
  author={Nichol, Alexander Quinn and Dhariwal, Prafulla},
  title={Improved denoising diffusion probabilistic models},
  booktitle={International conference on machine learning},
  pages={8162--8171},
  year={2021},
}

@inproceedings{LiTGLH22Continu1,
  author       = {Xiang Lisa Li and
                  John Thickstun and
                  Ishaan Gulrajani and
                  Percy Liang and
                  Tatsunori B. Hashimoto},
  title        = {Diffusion-LM Improves Controllable Text Generation},
  booktitle    = {Advances in Neural Information Processing Systems 35: Annual Conference on Neural Information Processing Systems, {NeurIPS}},
  year = {2022}
}

@inproceedings{AustinJHTB21Discret1,
  author       = {Jacob Austin and
                  Daniel D. Johnson and
                  Jonathan Ho and
                  Daniel Tarlow and
                  Rianne van den Berg},
  title        = {Structured Denoising Diffusion Models in Discrete State-Spaces},
  booktitle    = {Advances in Neural Information Processing Systems 34: Annual Conference
                  on Neural Information Processing Systems 2021, {NeurIPS}, 2021},
  pages        = {17981--17993},
  year         = {2021},
}

@inproceedings{HeSTWHQ23Discret,
  author       = {Zhengfu He and
                  Tianxiang Sun and
                  Qiong Tang and
                  Kuanning Wang and
                  Xuanjing Huang and
                  Xipeng Qiu},
  title        = {DiffusionBERT: Improving Generative Masked Language Models with Diffusion
                  Models},
  booktitle    = {Proceedings of the 61st Annual Meeting of the Association for Computational Linguistics (Volume 1: Long Papers), {ACL}},
  pages        = {4521--4534},
  year = {2023}
}

@inproceedings{GongLF0K23Continu,
  author       = {Shansan Gong and
                  Mukai Li and
                  Jiangtao Feng and
                  Zhiyong Wu and
                  Lingpeng Kong},
  title        = {DiffuSeq: Sequence to Sequence Text Generation with Diffusion Models},
  booktitle    = {The Eleventh International Conference on Learning Representations,
                  {ICLR}},
  year = {2023},
}

@misc{cobbe2021training,
  title={Training verifiers to solve math word problems},
  author={Cobbe, Karl and Kosaraju, Vineet and Bavarian, Mohammad and Chen, Mark and Jun, Heewoo and Kaiser, Lukasz and Plappert, Matthias and Tworek, Jerry and Hilton, Jacob and Nakano, Reiichiro and others},
  howpublished={arXiv preprint arXiv:2110.14168},
  year={2021}
}

@inproceedings{LightmanKBEBLLS24Math500,
  title        = {Let's Verify Step by Step},
  author       = {Hunter Lightman and
                  Vineet Kosaraju and
                  Yuri Burda and
                  Harrison Edwards and
                  Bowen Baker and
                  Teddy Lee and
                  Jan Leike and
                  John Schulman and
                  Ilya Sutskever and
                  Karl Cobbe},
  booktitle    = {The Twelfth International Conference on Learning Representations, {ICLR}},
  year={2024},
}

@inproceedings{WuZXLDZLHX26,
  title        = {Fast-dLLM: Training-free Acceleration of Diffusion {LLM} by Enabling
                  {KV} Cache and Parallel Decoding},
  author       = {Chengyue Wu and
                  Hao Zhang and
                  Shuchen Xue and
                  Zhijian Liu and
                  Shizhe Diao and
                  Ligeng Zhu and
                  Ping Luo and
                  Song Han and
                  Enze Xie},
  booktitle    = {The Fourteenth International Conference on Learning Representations, {ICLR}},
  year         = {2026},
}

@inproceedings{QiDZWJZ26,
  title        = {Hierarchy Decoding: {A} Training-free Parallel Decoding Strategy for
                  Diffusion Large Language Models},
  author       = {Xiaojing Qi and
                  Lun Du and
                  Xinyuan Zhang and
                  Lanning Wei and
                  Tao Jin and
                  Da Zheng},
  booktitle    = {The Fourteenth International Conference on Learning Representations, {ICLR}},
  year         = {2026},
}

@inproceedings{HongYYHZZWY26,
  title        = {Wide-In, Narrow-Out: Revokable Decoding for Efficient and Effective
                  DLLMs},
  author       = {Feng Hong and
                  Geng Yu and
                  Yushi Ye and
                  Haicheng Huang and
                  Huangjie Zheng and
                  Ya Zhang and
                  Yanfeng Wang and
                  Jiangchao Yao},
  booktitle    = {The Fourteenth International Conference on Learning Representations, {ICLR}},
  year         = {2026},
}

@misc{chen2021humaneval,
      title={Evaluating Large Language Models Trained on Code}, 
      author={Mark Chen and Jerry Tworek and Heewoo Jun and Qiming Yuan and Henrique Ponde de Oliveira Pinto and Jared Kaplan and Harri Edwards and Yuri Burda and Nicholas Joseph and Greg Brockman and Alex Ray and Raul Puri and Gretchen Krueger and Michael Petrov and Heidy Khlaaf and Girish Sastry and Pamela Mishkin and Brooke Chan and Scott Gray and Nick Ryder and Mikhail Pavlov and Alethea Power and Lukasz Kaiser and Mohammad Bavarian and Clemens Winter and Philippe Tillet and Felipe Petroski Such and Dave Cummings and Matthias Plappert and Fotios Chantzis and Elizabeth Barnes and Ariel Herbert-Voss and William Hebgen Guss and Alex Nichol and Alex Paino and Nikolas Tezak and Jie Tang and Igor Babuschkin and Suchir Balaji and Shantanu Jain and William Saunders and Christopher Hesse and Andrew N. Carr and Jan Leike and Josh Achiam and Vedant Misra and Evan Morikawa and Alec Radford and Matthew Knight and Miles Brundage and Mira Murati and Katie Mayer and Peter Welinder and Bob McGrew and Dario Amodei and Sam McCandlish and Ilya Sutskever and Wojciech Zaremba},
      howpublished={arXiv preprint arXiv:2107.03374},
      year={2021},
}

@misc{austin2021mbpp,
      title={Program Synthesis with Large Language Models}, 
      author={Jacob Austin and Augustus Odena and Maxwell Nye and Maarten Bosma and Henryk Michalewski and David Dohan and Ellen Jiang and Carrie Cai and Michael Terry and Quoc Le and Charles Sutton},
      howpublished={arXiv preprint arXiv:2108.07732},
      year={2021},
}

@misc{zhou2023ifeval,
      title={Instruction-Following Evaluation for Large Language Models}, 
      author={Jeffrey Zhou and Tianjian Lu and Swaroop Mishra and Siddhartha Brahma and Sujoy Basu and Yi Luan and Denny Zhou and Le Hou},
      howpublished={arXiv preprint arXiv:2311.07911},
      year={2023},
}

@misc{KimN2026,
      title={The Confidence Shortcut: A Reasoning Failure Mode of Masked Diffusion Models}, 
      author={Dueun Kim and Albert No},
      howpublished={arXiv preprint arXiv:2605.29123},
      year={2026},
}

@misc{BansalCKSS2026,
      title={Walk fast but be careful: Understanding Parallel Sampling in Masked Diffusion}, 
      author={Vansh Bansal and Cholyeon Cho and Syamantak Kumar and Sujay Sanghavi and Purnamrita Sarkar},
      howpublished={arXiv preprint arXiv:2606.22976},
      year={2026},
}

@inproceedings{KimSKKC2025,
  author       = {Jaeyeon Kim and
                  Kulin Shah and
                  Vasilis Kontonis and
                  Sham M. Kakade and
                  Sitan Chen},
  editor       = {Aarti Singh and
                  Maryam Fazel and
                  Daniel Hsu and
                  Simon Lacoste{-}Julien and
                  Felix Berkenkamp and
                  Tegan Maharaj and
                  Kiri Wagstaff and
                  Jerry Zhu},
  title        = {Train for the Worst, Plan for the Best: Understanding Token Ordering in Masked Diffusions},
  booktitle    = {Forty-second International Conference on Machine Learning, {ICML}},
  volume       = {267},
  year         = {2025},
}

@inproceedings{ZhongGZLDJSLZLZ2026,
  author       = {Yangyang Zhong and
                  Yanmei Gu and
                  Zhengqing Zang and
                  Xiaomeng Li and
                  Yuqi Ding and
                  Xibei Jia and
                  Yuting Shen and
                  Zhenzhong Lan and
                  Liwang Zhu and
                  Weiping Liu and
                  Junlin Zhou and
                  Haisheng Liu and
                  Zhong Xin Yu and
                  Pengxin Luo and
                  Donglian Qi and
                  Yunfeng Yan and
                  Junbo Zhao},
  editor       = {Maria Liakata and
                  Viviane P. Moreira and
                  Jiajun Zhang and
                  David Jurgens},
  title        = {Parallelism and Generation Order in Masked Diffusion Language Models: Limits Today, Potential Tomorrow},
  booktitle    = {Findings of the Association for Computational Linguistics, {ACL}},
  pages        = {7178--7218},
  year         = {2026},
}

@InProceedings{LeeKPP2026,
  title = 	 {Lookahead Unmasking Elicits Accurate Decoding in Diffusion Language Models},
  author =       {Lee, Sanghyun and Kim, Seungryong and Park, Jongho and Park, Dongmin},
  booktitle = 	 {Proceedings of the 43rd International Conference on Machine Learning, {ICML}},
  pages = 	 {64158--64174},
  year = 	 {2026},
  volume = 	 {306},
}

@InProceedings{ChenCL2026,
  title = 	 {Optimal Inference Schedules for Masked Diffusion Models},
  author =       {Chen, Sitan and Cong, Kevin and Li, Jerry},
  booktitle = 	 {Proceedings of Thirty Ninth Conference on Learning Theory, {COLT}},
  pages = 	 {1279--1311},
  year = 	 {2026},
  volume = 	 {336},
}
\FloatBarrier
\appendix

\section{Proofs of the Round-Complexity Bounds}
\label{bpd:app:sec:proofs}

\subsection{Deterministically Balanced Pivots}

\begin{proof}[Proof of Proposition~\ref{bpd:main:prop:balanced}]
Let $L_t$ be the largest remaining masked-interval length 
after $t$ rounds, setting $L_t=0$ after completion. 
Initially $L_0=J$. 
Each active interval is processed once per round, 
and each child has length 
at most $\rho$ times its parent length. 
Therefore
\begin{equation}
    L_{t+1}\leq\rho L_t,
    \quad L_t\leq\rho^tJ.
\end{equation}
Choose $t_\star=\lfloor\log J/\log(1/\rho)\rfloor+1$. 
Then $\rho^{t_\star}J<1$. 
Since $L_{t_\star}$ is a nonnegative integer, it must equal zero. 
This proves Eq.~\eqref{bpd:main:eq:bound}. 
The argument also applies to any initial collection of masked intervals 
when $J$ is their maximum length, 
because all intervals advance in parallel.
\end{proof}

\paragraph{A discrete middle-half condition.}
For an interval $I=[a,b]$ of length $m=b-a+1$, define
\begin{equation}
    A(I)=\left\{j\in I\biggm| \max\{j-a,b-j\}\leq\frac{3m}{4}\right\}.
\end{equation}
Selecting any $j\in A(I)$ guarantees 
the desired $1/4$--$3/4$-type split, 
with pivot removal and integer rounding handled explicitly. 
Equivalently,
\begin{equation}
    \max\{j-a,b-j\}=\frac{m-1}{2}+\left|j-\frac{a+b}{2}\right|.
\end{equation}
Thus, if $E_t(I)\cap A(I)\neq\emptyset$, 
the eligible token nearest the center also belongs to $A(I)$. 
However, the practical threshold rule does not guarantee 
that this intersection is nonempty or 
that a fallback pivot is balanced.

\begin{corollary}[Exact midpoint depth]
\label{bpd:main:cor:midpoint}
Exact midpoint selection completes an initially masked interval of length $J$ 
in $\lceil\log_2(J+1)\rceil$ rounds, 
including $T(0)=0$.
\end{corollary}
\begin{proof}
For an interval of length $m\geq1$, 
midpoint selection leaves a largest child of length $\lfloor m/2\rfloor$. 
Consequently,
\begin{equation}
    T(m)=1+T(\lfloor m/2\rfloor),\quad T(0)=0.
\end{equation}
Repeated halving gives $T(m)=\lfloor\log_2 m\rfloor+1=\lceil\log_2(m+1)\rceil$ 
for integer $m\geq1$.
\end{proof}

\subsection{Expected Depth under Probabilistic Balance}

\begin{proposition}[Probabilistically balanced pivots]
\label{bpd:main:prop:probabilistic}
Start from an interval of length $J\geq1$, 
and commit one pivot in every active interval each round. 
Conditional on the entire history before that round, 
suppose each interval's pivot satisfies Eq.~\eqref{bpd:main:eq:balance} 
with probability at least $q>0$, 
where $q$ and $\rho<1$ are constants independent of $J$. 
No independence between intervals or rounds is required. 
Set $\alpha=1-q(1-\rho)$. Then
\begin{equation}
    \Pr(T>t)\leq\min\{1,J^2\alpha^t\},
\end{equation}
\begin{equation}
    \mathbb E[T]\leq
    \left\lceil\frac{2\log J}{-\log\alpha}\right\rceil
    +\frac{1}{1-\alpha}
    =O\!\left(\frac{\log(J+1)}{q(1-\rho)}\right).
\end{equation}
\end{proposition}
\begin{proof}
Let $\mathcal F_t$ denote the history after $t$ rounds 
and define the potential
\begin{equation}
    \Phi_t=\sum_{I\in\mathcal I_t}|I|^2.
\end{equation}
For an interval of length $m$, the child lengths $\ell,r$ satisfy $\ell+r=m-1$. On a balanced split,
\begin{equation}
    \ell^2+r^2\leq\max\{\ell,r\}(\ell+r)
    \leq\rho m(m-1)\leq\rho m^2.
\end{equation}
On any other split, $\ell^2+r^2\leq(\ell+r)^2\leq m^2$. The conditional balance assumption and linearity of expectation therefore imply
\begin{equation}
    \mathbb E[\Phi_{t+1}\mid\mathcal F_t]
    \leq\bigl(q\rho+1-q\bigr)\Phi_t=\alpha\Phi_t.
\end{equation}
Since $\Phi_0=J^2$, iteration gives $\mathbb E[\Phi_t]\leq J^2\alpha^t$. 
Decoding is unfinished exactly when $\Phi_t\geq1$, 
so Markov's inequality yields the stated tail bound. 
Let $t_0=\lceil2\log J/(-\log\alpha)\rceil$. 
Summing the tail probabilities gives
\begin{align}
    \mathbb E[T]
    &=\sum_{t=0}^{\infty}\Pr(T>t)\\
    &\leq t_0+\sum_{t=t_0}^{\infty}J^2\alpha^t
    \leq t_0+\frac{1}{1-\alpha}.
\end{align}
Finally, $-\log\alpha\geq1-\alpha=q(1-\rho)$ gives the asymptotic bound.
\end{proof}

\paragraph{Uniform random pivots.}
At least half the positions in any nonempty interval 
belong to $A(I)$. 
Indeed, with local ranks $r\in\{1,\ldots,m\}$, 
the condition is $\lceil m/4\rceil\leq r\leq\lfloor3m/4\rfloor+1$, 
intersected with $\{1,\ldots,m\}$. 
Thus uniform random pivoting satisfies the proposition 
with $q=1/2$ and $\rho=3/4$, giving expected logarithmic depth. 
This scheduling fact does not establish that random pivoting preserves generation quality. 
Nor does it imply that confidence-based pivots have 
the same conditional balance probability.

\subsection{Bulk Initialization and the Practical Rule}

Suppose the first $K$ iterations have completed and masks remain. 
Let $R$ be the largest remaining interval length. 
If every subsequent pivot is $\rho$-balanced, 
the deterministic result gives
\begin{equation}
    T_{\mathrm{total}}\leq K+
    \left\lfloor\frac{\log R}{\log(1/\rho)}\right\rfloor+1.
\end{equation}
If no masks remain, $T_{\mathrm{total}}\leq K$. 
Bulk commits can only delete positions and split intervals, 
so $R\leq B$. 
This is a conditional depth bound. 
It does not assert that bulk initialization always reduces 
the actual runtime relative to another schedule.

For the unrestricted confidence-based rule, 
every unfinished iteration commits 
at least one non-MASK token 
and never remasks a committed position. 
Hence $T_{\mathrm{total}}\leq B$. 
This bound can be attained.
If bulk candidates are absent 
and each interval's unique highest-confidence token 
is always its left endpoint, 
with all confidences below the pivot threshold, 
the rule commits just one endpoint per round. 
This gives $T(B)=B$ and shows 
why an unconditional logarithmic claim 
for the practical rule would be incorrect.

\section{Detailed Algorithms}
\label{bpd:app:sec:detailed_algorithms}

\textsc{Cache} denotes the construction of the KV cache 
for the prefix preceding the current block $J$.
For clarity, cache construction is shown separately. 
In practice, it shares a single model forward pass 
with the first \textsc{ForwardWithCache} call.
The resulting prefix cache remains fixed 
throughout decoding of $J$ 
and is refreshed at the next block boundary.
\textsc{ForwardWithCache} returns token probabilities over $J$,
excluding $\mathtt{MASK}$ from the predictive distribution.
Ties in positional selection are broken to the left.
We assume $B\mid G$.

\begin{algorithm*}[t]
\caption{Confidence-guided decoding with bulk initialization and center pivots}
\label{bpd:app:alg:center_pivot}
\begin{algorithmic}[1]
\Require Model $f_\theta$, prompt $u$, generation length $G$, block size $B$
\Require Bulk steps $K$, thresholds $\tau_{\mathrm{bulk}}$ and $\tau_{\mathrm{pivot}}$
\State $y \gets \mathtt{MASK}^{G}$
\For{each block $J$ of size $B$ in $y$, from left to right}
    \State $t \gets 0$
    \State $C \gets \Call{Cache}{f_\theta,u,y,J}$
    \While{$y_J$ contains masks}
        \State $p \gets \Call{ForwardWithCache}{f_\theta,u,y,J,C}$
        \State $\widehat{y}_i \gets \arg\max_v p_i(v)$
        \State $c_i \gets p_i(\widehat{y}_i)$ for each $i \in J$
        \State $M \gets \{i \in J : y_i=\mathtt{MASK}\}$ 
        \State $S \gets \emptyset$
        \If{$t<K$}
            \State $S \gets \{i \in M : c_i \geq \tau_{\mathrm{bulk}}\}$
        \EndIf
        \If{$S=\emptyset$}
            \For{each maximal contiguous interval $I=[a,b]$ in $M$}
                \State $E \gets \{i \in I : c_i \geq \tau_{\mathrm{pivot}}\}$
                \If{$E \neq \emptyset$}
                    \State $j \gets \arg\min_{i \in E}|i-(a+b)/2|$
                \Else
                    \State $j \gets \arg\max_{i \in I}c_i$
                \EndIf
                \State $S \gets S \cup \{j\}$
            \EndFor
        \EndIf
        \State $y_i \gets \widehat{y}_i$ for all $i \in S$ simultaneously
        \State $t \gets t+1$
    \EndWhile
\EndFor
\State \Return $y$
\end{algorithmic}
\end{algorithm*}

All selections in an iteration use 
the same pre-update model output.
During the first three iterations of each block, 
a nonempty bulk set replaces interval-wise selection.
If it is empty, the pivot rule applies immediately.
The pivot rule commits exactly one token 
per remaining masked interval ($n=1$).
Committed tokens are not remasked. 
Prefix KV is refreshed at each block boundary.

\section{Experimental Details}
\label{bpd:app:sec:experimental_details}

\subsection{Datasets, Metrics and Prompts}
\label{bpd:app:ssec:datasets_metrics_prompts}

\paragraph{Datasets.}
We evaluate on five benchmarks spanning mathematical reasoning,
code generation, and instruction following: GSM8K, MATH500,
HumanEval, MBPP, and IFEval. 
This selection follows the common practice of 
evaluating diffusion decoding methods 
on mathematical reasoning and code generation, 
as in Fast-dLLM, WINO, and Hierarchy-dLLM.
We additionally include IFEval to assess whether accelerated
decoding preserves adherence to explicit instructions.
Table~\ref{bpd:app:tab:datasets} summarizes the benchmarks and evaluation split sizes.

\begin{table*}[t]
    \centering
    \small
    \begin{tabularx}{\textwidth}{@{}lXr@{}}
        \toprule[1.2pt]
        Dataset & Description & Examples \\
        \midrule
        GSM8K~\citep{cobbe2021training}
            & Grade-school mathematical word problems
              requiring multistep reasoning.
            & 1,319 \\
        MATH500~\citep{LightmanKBEBLLS24Math500}
            & A 500-problem subset of MATH covering
              competition-level mathematics.
            & 500 \\
        HumanEval~\citep{chen2021humaneval}
            & Python function completion from signatures
              and docstrings, evaluated with unit tests.
            & 164 \\
        MBPP~\citep{austin2021mbpp}
            & Basic Python programming problems with
              natural-language specifications and tests.
            & 500 \\
        IFEval~\citep{zhou2023ifeval}
            & Instruction-following prompts with
              automatically verifiable constraints.
            & 541 \\
        \bottomrule[1.2pt]
    \end{tabularx}
    \caption{Evaluation benchmarks and split sizes.}
    \label{bpd:app:tab:datasets}
\end{table*}

\paragraph{Metrics.}
For GSM8K, we report numerical answer accuracy after
normalization. We extract the answer following the final
\texttt{\#\#\#\#} marker, falling back to the last number in the
response when the marker is absent.
For MATH500, we extract the final boxed answer and use normalized
string matching, allowing an absolute error below $10^{-6}$
when both answers can be parsed as numbers.
This evaluator does not check symbolic equivalence.

For HumanEval and MBPP, we report Pass@1 using one greedy
completion per problem. For $N$ evaluation examples,
we compute
\begin{equation}
    \mathrm{Pass@1} =
    \frac{1}{N}
    \sum_{x=1}^{N}
    \mathbf{1}\left[
        x\text{ passes all tests}
    \right].
\end{equation}
HumanEval and MBPP use their unit tests.
For IFEval, we report both prompt-level and instruction-level
accuracy under the original strict checking rules.
All decoding methods share the same task-specific evaluators.

We additionally measure decoding depth using
the mean number of model forward steps.
Let $S_{x,b}$ denote the number of forward passes
required to decode block $b$ of example $x$.
We compute
\begin{equation}
    \overline{S} =
    \frac{1}{N}
    \sum_{x=1}^{N}
    \sum_{b=1}^{G/B} S_{x,b}.
\end{equation}
All decoding iterations are included,
including confidence-based prefilling.
This metric measures sequential model evaluations
over the complete generated sequence.

\paragraph{Pivot-Balance Metrics.}
To examine the interval contraction property
underlying Proposition~\ref{bpd:main:prop:balanced},
we measure the balanced pivot rate (BPR).
For a pivot $j$ selected from an interval
$I=[a,b]$, define the larger-child ratio as
\begin{equation}
    r(I,j) =
    \frac{\max\{j-a,b-j\}}{|I|}.
\end{equation}
A pivot is balanced if $r(I,j)\leq3/4$.
For a collection of pivot events $\mathcal D$,
we define
\begin{equation}
    \mathrm{BPR}(\mathcal D) =
    \frac{1}{|\mathcal D|}
    \sum_{(I,j)\in\mathcal D}
    \mathbf{1}\left[
        r(I,j)\leq\frac34
    \right].
    \label{bpd:app:eq:bpr}
\end{equation}

Since active masked intervals are decoded concurrently,
the sequential decoding depth is determined by
the longest branch of the interval-splitting process.
For each block, we reconstruct the splitting tree
from the recorded mask transitions
and identify a longest root-to-leaf path,
breaking ties in favor of the leftmost path.
Let $\mathcal C$ denote the collection of
pivot events along these critical paths.
We define the critical-path balanced pivot rate
(CP-BPR) as
\begin{equation}
    \mathrm{CP\text{-}BPR} =
    \mathrm{BPR}(\mathcal C).
\end{equation}

For \ours{}, we measure pivot balance
after the first $K=3$ iterations of each block,
excluding the initial prefilling window.
Intervals of length at most four are excluded,
since every pivot in such intervals satisfies
the $3/4$-balance condition.
Only single-pivot commitments are counted.

For \ours{}, we further distinguish confidence-eligible
pivots, selected when $E_t(I)\neq\emptyset$,
from fallback pivots, selected when
$E_t(I)=\emptyset$.
Their respective BPR values are computed using
Eq.~\eqref{bpd:app:eq:bpr}
over the corresponding event sets.

\paragraph{Prompts.}
We use zero-shot prompts and apply each model's native chat
template with a single user message and an assistant generation
prefix. We retain the original problem statements and add
task-specific output instructions.
GSM8K requests a step-by-step solution ending with
\texttt{\#\#\#\# <answer>}, while MATH500 requests a step-by-step
solution with the final answer enclosed in \verb|\boxed{}|.
HumanEval requests a complete Python function in a Python code
block, given the original function signature and docstring.
MBPP supplies the task description and its provided tests,
and requests a complete solution in a Python code block.
IFEval uses the original prompt without additional task
instructions. Within each model, these prompts are identical
across decoding methods.

\subsection{Models}

We use LLaDA-8B-Instruct~\citep{nie2025llada}, LLaDA-1.5~\citep{ZhuWNZWZLWL2026} and
Dream-v0-Instruct-7B~\citep{ye2025dream}, with the public
checkpoints \texttt{GSAI-ML/LLaDA-8B-Instruct}, \texttt{GSAI-ML/LLaDA-1.5}, and
\texttt{Dream-org/Dream-v0-Instruct-7B}, respectively.
All three are open masked diffusion language models that support
iterative token prediction with bidirectional context,
making them suitable for evaluating alternative decoding
orders and parallel commitment strategies.
They also represent different training approaches:
LLaDA is pretrained from scratch with a masked diffusion
objective, whereas Dream is initialized from a pretrained
autoregressive language model and adapted to diffusion
modeling. Evaluating both allows us to assess whether the
proposed decoding strategy transfers across these model
families at a similar parameter scale.
These families also provide continuity with prior decoding
studies, including Fast-dLLM and Hierarchy-dLLM.
We use the instruction-tuned checkpoints to support the
shared conversational prompting protocol across all five
tasks. Model parameters remain frozen, and all methods
use greedy token predictions without additional training.

\subsection{Baselines}
\label{bpd:app:ssec:baselines}

We compare against six training-free decoding baselines
covering fixed-budget unmasking, confidence-based parallel
decoding, dependency-aware token selection,
iterative revision, and hierarchical interval decomposition.

\paragraph{Vanilla.}
Vanilla uses the standard confidence-based unmasking schedule
in our evaluation framework. Within each generation block,
it ranks masked positions by the probability of their predicted
tokens and commits the highest-confidence predictions according
to a predefined per-step token budget. The remaining positions
stay masked, and committed tokens remain fixed.
With one decoding step per generated token, this schedule
commits one token per iteration. Our implementation recomputes
the full sequence without KV caching.

\paragraph{WINO.}
Wide-In, Narrow-Out~\citep{HongYYHZZWY26} combines parallel
drafting with verification of previously committed tokens.
It accepts multiple sufficiently confident predictions and
uses an auxiliary masked copy of the current block to evaluate
committed tokens without exposing each token to its own
verifier prediction. Tokens with low verification confidence
can be remasked and reconsidered in subsequent iterations.
This allows the decoder to revise earlier commitments as
additional context becomes available.

\paragraph{Fast-dLLM.}
Fast-dLLM~\citep{WuZXLDZLHX26} combines confidence-aware parallel
decoding with blockwise KV-cache reuse. Its parallel selection
rule commits masked-token predictions whose confidence exceeds
a threshold, with a highest-confidence fallback to ensure
progress. The number of tokens committed therefore adapts to
model confidence. Its caching mechanism reuses context
representations to reduce the computation required by
subsequent decoding iterations.

\paragraph{Hierarchy-dLLM.}
Hierarchy-dLLM~\citep{QiDZWJZ26} recursively partitions
unresolved positions into contiguous sub-decoding areas and
processes these areas in parallel. It combines a high confidence
threshold for parallel commitments with a lower threshold for
selecting the most confident candidate within each area,
together with a fallback that ensures progress.
It also supports remasking low-confidence committed tokens.
We use its remasking variant without KV caching.
Its confidence-based selection within each area provides
a direct comparison for the center preference in our pivot rule.

\paragraph{DAPD}
Dependency-Aware Parallel Decoding~\cite{KimJJ26DPAD} introduced a dependency-aware parallel decoding strategy that modeled token dependencies through a graph constructed from self-attention scores. 
It represented masked positions as graph nodes and connected token pairs with strong dependencies. 
The method then identified an independent set of tokens for parallel unmasking to reduce interference among dependent predictions. 
This graph-based selection strategy aimed to improve decoding efficiency while preserving generation quality. 

\paragraph{DOS}
Dependency-Oriented Sampler~\citep{ZhouHH26DOS} proposed an attention-based decoding strategy that determined the unmasking order according to token dependencies on the observed context. 
It estimated the dependency score of each masked position by aggregating its attention weights toward unmasked tokens. 
The method prioritized positions with higher dependency scores to leverage the available contextual information during generation. 
This selection mechanism aimed to improve prediction reliability by determining the decoding order from contextual dependencies rather than token confidence alone. 

\paragraph{Bisection Variants.}
For the pivot-balance analysis, we consider four
balanced bisection strategies.
The first three are adapted from
\citet{BansalCKSS2026} to our decoding framework.

\emph{Bisection} recursively selects the midpoint
of each active masked interval.
\emph{Bisection (Margin)} and
\emph{Bisection (Entropy)} restrict pivot candidates
to the centered half of each interval,
following the score-guided bisection principle.
Margin selects the position with the largest difference
between the two highest token probabilities,
whereas Entropy selects the position
with the lowest predictive entropy.

We additionally introduce
\emph{Bisection (Random)}
to examine balanced pivot selection
without predictive confidence.
For each interval $I=[a,b]$,
we define the balanced region
\begin{equation}
    A_{\rho}(I)=
    \left\{
    j\in I \;\middle|\;
    \max\{j-a,b-j\}\leq\rho|I|
    \right\},
\end{equation}
where $\rho=3/4$.
At each iteration, this variant samples
one pivot uniformly from the balanced region:
\begin{equation}
    j_t(I)\sim\operatorname{Uniform}(A_{\rho}(I)).
\end{equation}
Unlike the uniform random pivoting analyzed
in Appendix~\ref{bpd:app:sec:proofs},
this variant restricts sampling to the
balanced region and therefore guarantees
the balance condition at every split.

All four bisection variants select one pivot
per active masked interval at each iteration
and commit the corresponding greedy predictions
in parallel.
Following the prefilling phase, 
their pivots satisfy the $3/4$-balance condition
by construction.

\subsection{Hyperparameter Settings}
\label{bpd:app:ssec:hyperparameter_settings}

Unless otherwise specified, we use a generation length
of $G=256$ and a block size of $B=32$.
\ours{} uses $K=3$ initial prefilling iterations,
with confidence thresholds
$\tau_{\mathrm{bulk}}=\tau_{\mathrm{pivot}}=0.9$.
All methods use greedy token predictions.

\paragraph{Pivot Selection Comparison.}
We compare five pivot selection strategies
on HumanEval using LLaDA-8B-Instruct,
with 164 evaluation problems.
All methods use identical prompts,
a generation length of 256,
a block size of 32, and prefix KV caching.
Each strategy is evaluated with and without
confidence-based prefilling.
The prefilling-enabled configurations use
$K=3$ and $\tau_{\mathrm{bulk}}=0.9$.
\ours{} additionally uses
$\tau_{\mathrm{pivot}}=0.9$ for
confidence-thresholded pivot selection.

Exact, Margin, and Entropy follow
the bisection strategies of \citet{BansalCKSS2026},
while Random is our additional balanced-region
sampling baseline.
Unlike these four strategies, which enforce
the $3/4$-balance condition during pivot selection,
\ours{} uses confidence-thresholded center selection
with a highest-confidence fallback.
Further implementation details are provided in
Appendix~\ref{bpd:app:ssec:baselines}.

\subsection{Experimental Environments}
All experiments were conducted on a server equipped with
an NVIDIA RTX A6000 GPU (48 GB) and AMD EPYC 9354 32-Core CPU,
running Rocky Linux 9.4 (Blue Onyx), kernel 5.14.0-427.22.1e9\_4.x86\_64.
The software environment used Python 3.11.0,
PyTorch 2.6.0+cu124, and CUDA 12.4.
The main dependencies included Accelerate 1.15.0,
Datasets 5.0.1, Scikit-learn 1.9.1,
and Transformers 4.52.3.

\section{Additional Experimental Results}

\subsection{Additional Baseline Configurations}
\label{bpd:app:ssec:baseline_configurations}


\begin{table*}[t!]
\centering
\small
\setlength{\tabcolsep}{4.5pt}
\begin{tabular}{llccccccccc}
\toprule
\multirow{2}{*}[-2pt]{Dataset}
& \multirow{2}{*}[-2pt]{Method}
& \multicolumn{3}{c}{LLaDA-8B-Instruct}
& \multicolumn{3}{c}{LLaDA-1.5}
& \multicolumn{3}{c}{Dream-7B} \\
\cmidrule(lr){3-5}
\cmidrule(lr){6-8}
\cmidrule(lr){9-11}
& & Score $\uparrow$
  & TPS $\uparrow$
  & QTG $\uparrow$
  & Score $\uparrow$
  & TPS $\uparrow$
  & QTG $\uparrow$
  & Score $\uparrow$
  & TPS $\uparrow$
  & QTG $\uparrow$ \\
\midrule

\multirow{8}{*}{GSM8K}
& Vanilla
& 81.65 & 12.02 & 1.00$\times$
& 81.88 & 11.92 & 1.00$\times$
& 82.87 & 14.62 & 1.00$\times$ \\
& Fast-dLLM~(prefix)
& 81.50 & 52.83 & 4.39$\times$
& 81.58 & 54.96 & 4.59$\times$
& 80.52 & 74.60 & 4.96$\times$ \\
& Fast-dLLM~(dual)
& 80.36 & 60.37 & \underline{4.94}$\times$
& 82.03 & 62.16 & \underline{5.22}$\times$
& 82.11 & 81.74 & \underline{5.54}$\times$ \\
& Fast-dLLM~(parallel)
& 82.18 & 35.50 & 2.97$\times$
& 82.03 & 36.41 & 3.06$\times$
& 34.19 & 26.46 & 0.75$\times$ \\
& DOS~(top-1)
& 80.36 & 12.04 & 0.99$\times$
& 82.64 & 12.04 & 1.02$\times$
& 82.71 & 14.84 & 1.01$\times$ \\
& DOS~(eb)
& 79.98 & 22.32 & 1.82$\times$
& 82.18 & 23.39 & 1.97$\times$
& 82.49 & 32.75 & 2.23$\times$ \\
& \CCG{}Ours~(prefix)
& \CCG{}79.98 & \CCG{}54.50
& \CCG{}4.44$\times$
& \CCG{}80.59 & \CCG{}56.10
& \CCG{}4.63$\times$
& \CCG{}79.91 & \CCG{}82.87
& \CCG{}5.47$\times$ \\
& \CCG{}Ours~(dual)
& \CCG{}78.92 & \CCG{}67.26
& \CCG{} \textbf{5.41}$\times$
& \CCG{}80.06 & \CCG{}68.64
& \CCG{} \textbf{5.63}$\times$
& \CCG{}79.45 & \CCG{}90.30
& \CCG{} \textbf{5.92}$\times$ \\
\midrule

\multirow{8}{*}{MATH-500}
& Vanilla
& 38.60 & 12.13 & 1.00$\times$
& 37.40 & 12.07 & 1.00$\times$
& 42.00 & 14.70 & 1.00$\times$ \\
& Fast-dLLM~(prefix)
& 37.00 & 41.06 & 3.24$\times$
& 36.20 & 42.02 & 3.37$\times$
& 42.60 & 54.40 & 3.75$\times$ \\
& Fast-dLLM~(dual)
& 36.40 & 44.52 & 3.46$\times$
& 37.00 & 46.51 & \underline{3.81}$\times$
& 42.00 & 56.76 & 3.86$\times$ \\
& Fast-dLLM~(parallel)
& 37.60 & 25.69 & 2.06$\times$
& 38.00 & 26.43 & 2.22$\times$
& 15.60 & 20.44 & 0.52$\times$ \\
& DOS~(top-1)
& 35.80 & 12.10 & 0.93$\times$
& 36.60 & 11.95 & 0.97$\times$
& 41.20 & 14.84 & 0.99$\times$ \\
& DOS~(eb)
& 34.60 & 17.27 & 1.28$\times$
& 35.80 & 18.05 & 1.43$\times$
& 39.60 & 21.19 & 1.36$\times$ \\
& \CCG{}Ours~(prefix)
& \CCG{}37.00 & \CCG{}44.20
& \CCG{}\underline{3.49}$\times$
& \CCG{}37.00 & \CCG{}44.95
& \CCG{}3.68$\times$
& \CCG{}40.40 & \CCG{}62.52
& \CCG{}\underline{4.09}$\times$ \\
& \CCG{}Ours~(dual)
& \CCG{}36.00 & \CCG{}53.91
& \CCG{} \textbf{4.14}$\times$
& \CCG{}35.20 & \CCG{}54.85
& \CCG{} \textbf{4.28}$\times$
& \CCG{}40.40 & \CCG{}63.48
& \CCG{} \textbf{4.15}$\times$ \\
\midrule

\multirow{8}{*}{HumanEval}
& Vanilla
& 35.98 & 10.44 & 1.00$\times$
& 39.63 & 10.34 & 1.00$\times$
& 59.76 & 11.79 & 1.00$\times$ \\
& Fast-dLLM~(prefix)
& 39.02 & 79.06 & \textbf{8.21$\times$}
& 38.41 & 78.54 & \textbf{7.36$\times$}
& 60.37 & 91.39 & 7.83$\times$ \\
& Fast-dLLM~(dual)
& 35.98 & 79.05 & 7.57$\times$
& 38.41 & 76.53 & 7.17$\times$
& 57.93 & 111.88 & \textbf{9.20}$\times$ \\
& Fast-dLLM~(parallel)
& 34.76 & 48.71 & 4.51$\times$
& 38.41 & 44.85 & 4.20$\times$
& 29.88 & 27.60 & 1.17$\times$ \\
& DOS~(top-1)
& 35.98 & 10.38 & 0.99$\times$
& 39.02 & 10.40 & 0.99$\times$
& 57.93 & 11.85 & 0.97$\times$ \\
& DOS~(eb)
& 34.76 & 26.93 & 2.49$\times$
& 37.80 & 27.76 & 2.56$\times$
& 57.32 & 34.08 & 2.77$\times$ \\
& \CCG{}Ours~(prefix)
& \CCG{}35.98 & \CCG{}78.29
& \CCG{}7.50$\times$
& \CCG{}37.20 & \CCG{}77.83
& \CCG{}7.07$\times$
& \CCG{}54.88 & \CCG{}107.61
& \CCG{}\underline{8.38$\times$} \\
& \CCG{}Ours~(dual)
& \CCG{}34.76 & \CCG{}87.95
& \CCG{} \underline{8.14}$\times$
& \CCG{}34.15 & \CCG{}86.95
& \CCG{} \underline{7.25}$\times$
& \CCG{}53.66 & \CCG{}120.78
& \CCG{} \textbf{9.20}$\times$ \\
\midrule

\multirow{8}{*}{MBPP}
& Vanilla
& 39.60 & 10.78 & 1.00$\times$
& 40.00 & 10.74 & 1.00$\times$
& 58.00 & 12.69 & 1.00$\times$ \\
& Fast-dLLM~(prefix)
& 38.00 & 52.01 & 4.63$\times$
& 37.80 & 52.89 & 4.65$\times$
& 57.60 & 112.33 & 8.79$\times$ \\
& Fast-dLLM~(dual)
& 37.60 & 58.82 & \underline{5.18}$\times$
& 36.80 & 58.60 & 5.02$\times$
& 54.80 & 134.80 & \underline{10.04}$\times$ \\
& Fast-dLLM~(parallel)
& 39.20 & 30.23 & 2.78$\times$
& 39.80 & 31.37 & 2.91$\times$
& 44.80 & 26.38 & 1.61$\times$ \\
& DOS~(top-1)
& 36.80 & 10.78 & 0.93$\times$
& 40.20 & 10.79 & 1.01$\times$
& 57.60 & 12.46 & 0.98$\times$ \\
& DOS~(eb)
& 40.20 & 20.58 & 1.94$\times$
& 39.80 & 21.28 & 1.97$\times$
& 57.80 & 54.23 & 4.26$\times$ \\
& \CCG{}Ours~(prefix)
& \CCG{}38.20 & \CCG{}55.90
& \CCG{}5.00$\times$
& \CCG{}38.60 & \CCG{}56.69
& \CCG{}\underline{5.09}$\times$
& \CCG{}53.00 & \CCG{}131.36
& \CCG{}9.46$\times$ \\
& \CCG{}Ours~(dual)
& \CCG{}35.00 & \CCG{} 69.09
& \CCG{} \textbf{5.66}$\times$
& \CCG{}37.20 & \CCG{} 69.00
& \CCG{} \textbf{5.97}$\times$
& \CCG{}52.80 & \CCG{} 147.81
& \CCG{} \textbf{10.60}$\times$ \\
\midrule

\multirow{8}{*}{IFEval}
& Vanilla
& 56.01 & 12.26 & 1.00$\times$
& 58.60 & 12.98 & 1.00$\times$
& 51.39 & 16.24 & 1.00$\times$ \\
& Fast-dLLM~(prefix)
& 53.79 & 30.09 & 2.36$\times$
& 57.49 & 31.41 & 2.37$\times$
& 52.31 & 47.52 & 2.98$\times$ \\
& Fast-dLLM~(dual)
& 54.34 & 39.07 & 3.09$\times$
& 56.56 & 36.24 & 2.69$\times$
& 51.02 & 53.67 & 3.28$\times$ \\
& Fast-dLLM~(parallel)
& 55.82 & 22.00 & 1.79$\times$
& 57.86 & 20.78 & 1.58$\times$
& 23.11 & 42.58 & 1.18$\times$ \\
& DOS~(top-1)
& 60.63 & 12.72 & 1.12$\times$
& 59.15 & 12.96 & 1.01$\times$
& 56.38 & 16.15 & 1.09$\times$ \\
& DOS~(eb)
& 59.89 & 19.77 & 1.72$\times$
& 60.44 & 19.52 & 1.55$\times$
& 56.93 & 25.17 & 1.72$\times$ \\
& \CCG{}Ours~(prefix)
& \CCG{}51.20 & \CCG{}43.66
& \CCG{}\underline{3.26}$\times$
& \CCG{}54.71 & \CCG{}43.48
& \CCG{}\underline{3.13}$\times$
& \CCG{}51.57 & \CCG{}63.94
& \CCG{}\underline{3.95}$\times$ \\
& \CCG{}Ours~(dual)
& \CCG{}51.02 & \CCG{} 55.78
& \CCG{} \textbf{4.14}$\times$
& \CCG{}54.53 & \CCG{} 54.21
& \CCG{} \textbf{3.89}$\times$
& \CCG{}51.39 & \CCG{} 70.51
& \CCG{} \textbf{4.34}$\times$ \\
\bottomrule
\end{tabular}
\caption{
Additional results on five benchmarks using LLaDA-8B-Instruct, LLaDA-1.5, and Dream-7B.
We report generation quality (Score), throughput measured
in tokens per second (TPS), and Quality-Throughput Gain (QTG).
QTG is the product of relative throughput and relative
task performance with respect to vanilla.
}
\label{bpd:app:tab:additional_results}
\end{table*}

We further evaluate multiple configurations of
Fast-dLLM, DOS, and \ours{} across five benchmarks
and three diffusion language models, as reported in
Table~\ref{bpd:app:tab:additional_results}.
For Fast-dLLM, we consider prefix, dual, and
parallel configurations, while DOS is evaluated
with top-1 and entropy-bounded (eb) selection.
We also evaluate \ours{} with both prefix and
dual caching to examine the effect of different
caching configurations under the same pivot
selection strategy.

Dual caching generally improves throughput,
but often comes at the cost of generation quality.
For \ours{}, dual caching increases throughput
and QTG in all 15 settings, while reducing task
scores in 14 settings and preserving them in one.
For example, on LLaDA-8B-Instruct with MBPP,
dual caching improves throughput from 55.90 to
69.09 TPS, while Pass@1 decreases from
38.20\% to 35.00\%.
We therefore adopt prefix caching for both
Fast-dLLM and \ours{} in the main experiments,
prioritizing generation quality over the
additional acceleration offered by dual caching.
For DOS, we use the entropy-bounded configuration
as its accelerated decoding variant.

Nevertheless, \ours{} remains competitive under
the more aggressive dual caching configuration.
It achieves the highest QTG in 12 of the 15
settings, with a near tie against Fast-dLLM (dual)
on Dream HumanEval.
The remaining two settings are HumanEval on
the LLaDA models, where Fast-dLLM (prefix)
achieves higher QTG.
These results show that the effectiveness of
\ours{} is not limited to prefix caching,
and that further throughput improvements are
possible when modest reductions in generation
quality are acceptable.

\begin{table*}[t]
\centering
\small
\setlength{\tabcolsep}{6pt}
\renewcommand{\arraystretch}{1.15}
\begin{tabular}{l r cc cc cc}
\toprule
\multirow{2}{*}[-2pt]{Category}
& \multirow{2}{*}[-2pt]{\#Inst.}
& \multicolumn{2}{c}{Accuracy (\%) $\uparrow$}
& \multicolumn{2}{c}{Early Term. (\%) $\downarrow$}
& \multicolumn{2}{c}{TPS~(tok/s) $\uparrow$} \\
\cmidrule(lr){3-4}
\cmidrule(lr){5-6}
\cmidrule(l){7-8}
& & Hier. & Ours & Hier. & Ours & Hier. & Ours \\
\midrule
startend             & 67  & 77.6 & \textbf{83.6} & 60.0 & \textbf{9.1} & 32.42 & \textbf{42.11} \\
detectable\_content  & 53  & 45.3 & \textbf{67.9} & 34.5 & \textbf{0.0} & 26.83 & \textbf{37.20} \\
keywords             & 163 & 57.7 & 57.7 & 29.0 & \textbf{7.2} & 31.82 & \textbf{41.88} \\
combination          & 65  & 41.5 & \textbf{66.2} & 28.9 & \textbf{22.7} & 49.21 & \textbf{60.11} \\
detectable\_format   & 157 & 63.7 & \textbf{72.6} & 22.8 & \textbf{4.7} & 31.11 & \textbf{43.68} \\
length\_constraints  & 143 & \textbf{46.9} & 42.0 & 21.1 & \textbf{2.4} & 27.16 & \textbf{36.43} \\
language             & 31  & 77.4 & \textbf{80.6} & \textbf{14.3} & 16.7 & 36.30 & \textbf{44.33} \\
change\_case         & 89  & \textbf{61.8} & 57.3 & \textbf{5.9} & 7.9 & 31.66 & \textbf{43.59} \\
punctuation          & 66  & \textbf{45.5} & 37.9 & \textbf{2.8} & 4.9 & 39.97 & \textbf{50.99} \\
\midrule
Overall & 834
& 56.7 & \textbf{60.4}
& 23.0 & \textbf{6.4}
& 32.83 & \textbf{43.66} \\
\bottomrule
\end{tabular}
\caption{
Instruction-level analysis of IFEval across instruction categories.
We report instruction-following accuracy, the proportion of failures attributable to early termination, and decoding TPS~(tokens/second) for Hierarchy-dLLM and our method. We define early termination as a failed instruction whose response length falls below 50\% of the corresponding vanilla response length.
}
\label{bpd:app:tab:ifeval_detailed}
\end{table*}

\begin{table*}[!ht]
\centering
\small
\setlength{\tabcolsep}{6pt}
\renewcommand{\arraystretch}{1.2}

\begin{tabularx}{\textwidth}{@{}X X@{}}
\toprule
\multicolumn{2}{@{}l}{\textbf{Case 1: Keyword and Repetition Constraints}} \\
\midrule
\multicolumn{2}{@{}p{\textwidth}@{}}{
\textbf{Prompt:}
Write a funny haiku about moms, containing keywords
``mom'' and ``mother'' in your response.
First repeat the request word for word without change,
then give your answer (...)
} \\
\midrule
\textbf{Hierarchy (16 tokens, Fail)}
&
\textbf{Ours (42 tokens, Pass)} \\
\midrule
Write a funny haiku about moms,
containing keywords "mom" and "mother
&
Write a funny haiku about moms, containing keywords "mom" and "mother" in your response.<br><br>Mom's smile is loud,<br>Mother's laugh is loud,<br>Laughter all around.
\\
\midrule
\multicolumn{2}{@{}l}{\textbf{Case 2: Explicit Ending Constraint}} \\
\midrule
\multicolumn{2}{@{}p{\textwidth}@{}}{
\textbf{Prompt:}
How to write a good Chinese poem?
At the end of your response, please explicitly
add a note starting with ``P.S.''
} \\
\midrule
\textbf{Hierarchy (59 tokens, Fail)}
&
\textbf{Ours (254 tokens, Pass)} \\
\midrule
Writing a good Chinese poem requires a blend of creativity, structure ( and understanding of traditional Chinese poetry. Here are some tips:<br><br>1. Understand the Structure: Familiarize yourself with the different forms of Chinese poetry, such as the 
jue 
\begin{CJK}{UTF8}{gbsn}
(绝句), 
\end{CJK}
lishi
\begin{CJK}{UTF8}{gbsn}
(绝句
\end{CJK}
&
Writing a good Chinese poem requires a blend of creativity, structure, and cultural understanding. Here are some tips:<br><br>1. Choose Traditional Themes: … <br>…<br>6. Practice and Reflect: …<br><br>P.S. Remember to stay true to the essence of traditional Chinese poetry while adding your own unique voice.
\\
\bottomrule
\end{tabularx}
\caption{
Qualitative comparison between Hierarchy-dLLM
and our method on IFEval.
Our method satisfies the specified instructions,
whereas Hierarchy-dLLM produces incomplete responses.
Ellipses indicate omitted text.
}
\label{bpd:app:tab:ifeval_case_study}
\end{table*}

\subsection{Fine-grained Analysis on Instruction Following}

Aggressive parallel decoding increases efficiency by committing multiple tokens at each decoding step. 
However, instruction-following tasks expose a distinct failure mode in which accelerated decoding terminates generation before the response satisfies the requested structure.
We analyze this behavior by decomposing IFEval results by instruction category.
We define an early termination error as a failed instruction whose response contains fewer than half as many tokens as the vanilla response to the same prompt.
We compare our method with Hierarchy-dLLM, which adopts a similar divide-and-conquer decoding structure, to examine how different token selection strategies affect instruction-following performance.
Table~\ref{bpd:app:tab:ifeval_detailed} shows that early termination accounts for 23.0\% of the failures from Hierarchy-dLLM, compared with 6.4\% from our method.

Completion-dependent instructions show the largest difference between the two methods.
Instructions such as startend, detectable\_content, and combination require the model to preserve generation until a specific part of the response appears.
Early termination prevents the decoder from producing required suffixes, postscripts, or content following a repeated request.
Table~\ref{bpd:app:tab:ifeval_detailed} shows that 60.0\% of the startend failures from Hierarchy-dLLM correspond to early termination.
Table~\ref{bpd:app:tab:ifeval_case_study} further illustrates cases in which Hierarchy-dLLM stops before generating the required content, while our method completes the response and satisfies the instruction.
Our method generates 167.9 tokens on average, which remains close to the 173.7 tokens from vanilla decoding.
Hierarchy-dLLM generates 146.6 tokens on average.

Our method achieves higher throughput across all instruction categories, with an overall improvement from 32.83 to 43.66 tokens per second. 
Despite this acceleration, our method maintains an average response length close to vanilla decoding (167.9 vs. 173.7 tokens), whereas Hierarchy-dLLM produces shorter responses (146.6 tokens). 
This difference also appears in the preservation of correct responses.
Our method retains 252 of the 303 prompts that vanilla decoding answers correctly, compared with 227 for Hierarchy-dLLM. 
These results demonstrate that our method improves decoding efficiency while preserving response completeness and instruction-following performance.

\subsection{Case Study of Pivot Selection}
\label{bpd:app:ssec:case_study}

\begin{figure*}[!t]
    \centering
    \includegraphics[width=\linewidth]{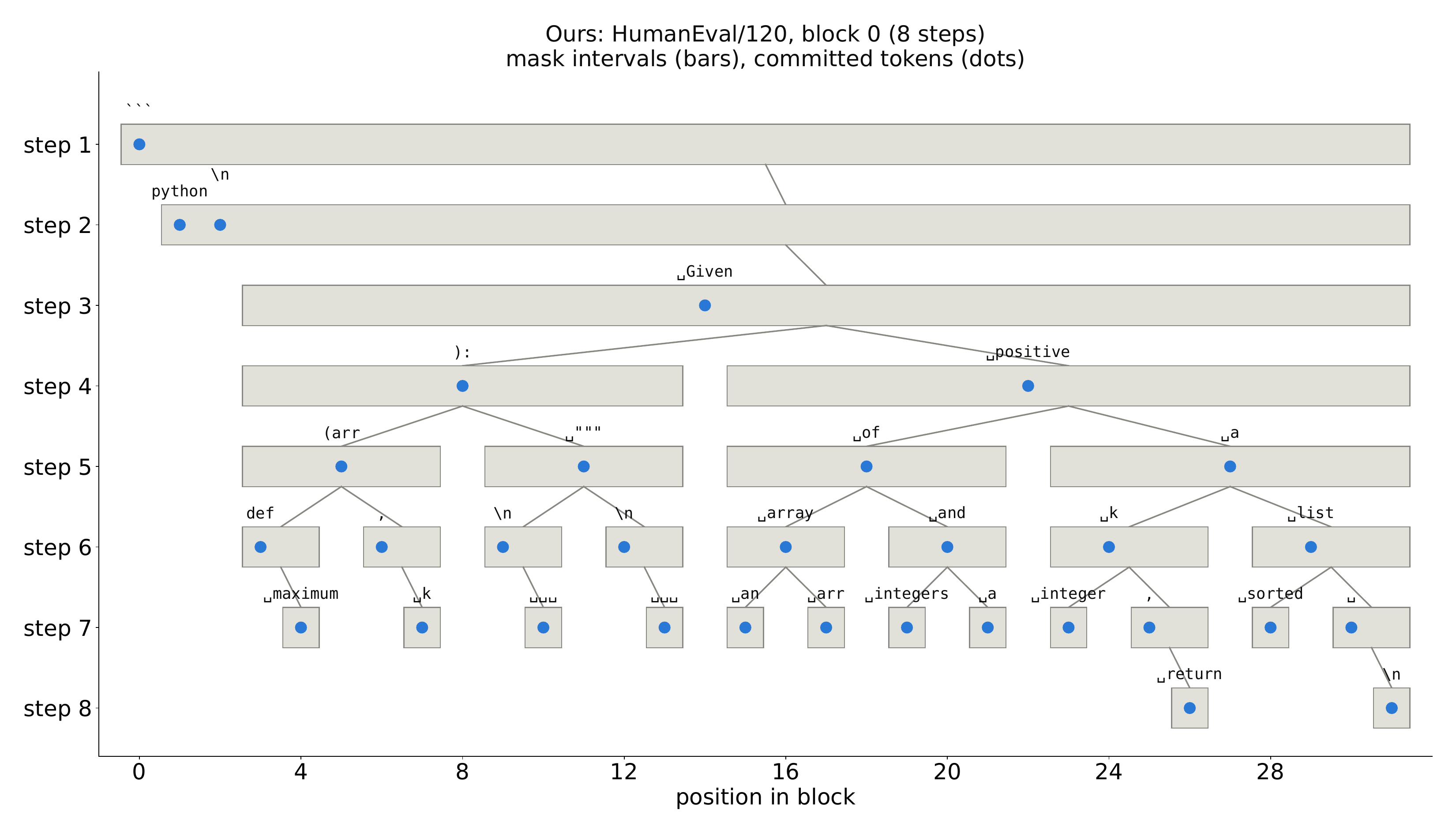}
    \caption{Decoding trajectory on HumanEval. Gray bars represent masked intervals, and 
    blue dots indicate committed tokens.
    Our method recursively partitions masked intervals through center-preferred pivot selection, completing a 32-token block in eight decoding steps.
    }
    \label{bpd:app:fig:decoding_tree}
\end{figure*}

Figure~\ref{bpd:app:fig:decoding_tree} illustrates
the token unmasking process of our method on a
HumanEval example using LLaDA-8B-Instruct.
The decoding process begins with confidence-based
token commitments that provide initial context
for subsequent predictions.
These early commitments create multiple masked
intervals of varying lengths across the block.
Our method then selects pivots near the center
of each remaining interval and recursively
partitions them into smaller subintervals.
Since each active interval is processed independently,
multiple pivots can be committed in parallel
during a single forward pass.
As decoding progresses, the remaining masked
intervals become shorter and more distributed
across the sequence.
The 32-token block is completed in only eight
decoding steps, compared with 32 steps required
by single-token unmasking.
This example illustrates how confidence-based
prefilling and interval-wise pivot selection
jointly exploit parallelism to reduce decoding depth.

\end{document}